%% file: main.tex
\documentclass[acmsmall]{acmart}

\usepackage{algorithmic}
\usepackage{graphicx}
\usepackage{textcomp}
\usepackage{caption}
\usepackage{xcolor}
\usepackage{subcaption}
\usepackage{booktabs}
\usepackage{wrapfig}
\usepackage{adjustbox}
\usepackage[pagewise]{lineno}

\usepackage{color, soul}
\usepackage[most]{tcolorbox}
\newtcolorbox{highlighted}{colback=yellow,breakable}

\usepackage[linesnumbered,ruled,vlined]{algorithm2e}
\SetKwInput{KwInput}{Input}
\SetKwInput{KwOutput}{Output}

\newtheorem{definition}{Definition}

\newtheorem{thm}{Theorem}

\definecolor{tamercolor}{rgb}{0.5, 0.0, 0.13}
\definecolor{forest}{rgb}{0.25, 0.74, 0.0}

\newcommand{\NAMEA}{NaiveRL}
\newcommand{\NAMEB}{CHARME}

\AtBeginDocument{%
  }

\setcopyright{acmlicensed}
\copyrightyear{XXXX}
\acmYear{XXXX}
\acmDOI{XXXXXXX.XXXXXXX}

\acmJournal{TQC}
\acmVolume{37}
\acmNumber{4}
\acmArticle{111}
\acmMonth{8}

\begin{document}

\title{CHARME: A Chain-Based Reinforcement Learning Approach for the Minor Embedding Problem}

\author{Hoang M. Ngo}
\authornote{Equal contribution.}
\orcid{1234-5678-9012}
\affiliation{%
  \institution{Department of Computer and Information Science and Engineering, University of Florida}
  \city{Gainesville}
  \state{Florida}
  \country{USA}
}

\author{Nguyen H K. Do}
\authornotemark[1]
\orcid{1234-5678-9012}
\affiliation{%
  \institution{Department of Computer and Information Science and Engineering, University of Florida}
  \city{Gainesville}
  \state{Florida}
  \country{USA}
}
\author{Minh N. Vu}
\authornote{Most of the work for this submission was completed while he was at the University of Florida.}
\orcid{1234-5678-9012}
\affiliation{%
  \institution{Department of Computer and Information Science and Engineering, University of Florida}
  \city{Gainesville}
  \state{Florida}
  \country{USA}
}
\affiliation{%
  \institution{Theoritical Division, Los Alamos National Laboratory}
  \city{Los Alamos}
  \state{New Mexico}
  \country{USA}
}
\author{Tre' R. Jeter}
\orcid{1234-5678-9012}
\affiliation{%
  \institution{Department of Computer and Information Science and Engineering, University of Florida}
  \city{Gainesville}
  \state{Florida}
  \country{USA}
}

\author{Tamer Kahveci}
\orcid{1234-5678-9012}
\affiliation{%
  \institution{Department of Computer and Information Science and Engineering, University of Florida}
  \city{Gainesville}
  \state{Florida}
  \country{USA}
}
\author{My T. Thai}
\authornote{Corresponding author.}
\orcid{1234-5678-9012}
\affiliation{%
  \institution{Department of Computer and Information Science and Engineering, University of Florida}
  \city{Gainesville}
  \state{Florida}
  \country{USA}
}

\renewcommand{\shortauthors}{Hoang Ngo et al.}

\begin{abstract}
  Quantum annealing (QA) has great potential to solve combinatorial optimization problems efficiently. However, the effectiveness of QA algorithms is heavily based on the embedding of problem instances, represented as logical graphs, into the quantum processing unit (QPU) whose topology is in the form of a limited connectivity graph, known as the minor embedding problem. Because the minor embedding problem is an NP-hard problem~\mbox{\cite{Goodrich2018}}, existing methods for the minor embedding problem suffer from scalability issues when faced with larger problem sizes. In this paper, we propose a novel approach utilizing Reinforcement Learning (RL) techniques to address the minor embedding problem, named \NAMEB. \NAMEB\  includes three key components: a Graph Neural Network (GNN) architecture for policy modeling, a state transition algorithm that ensures solution validity, and an order exploration strategy for effective training. Through comprehensive experiments on synthetic and real-world instances, we demonstrate the efficiency of our proposed order exploration strategy as well as our proposed RL framework, \mbox{\NAMEB}. In particular, \mbox{\NAMEB } yields superior solutions in terms of qubit usage compared to fast embedding methods such as Minorminer and ATOM. Moreover, our method surpasses the OCT-based approach, known for its slower runtime but high-quality solutions, in several cases. In addition, our proposed exploration enhances the efficiency of the training of the \NAMEB\ framework by providing better solutions compared to the greedy strategy.
\end{abstract}



\maketitle

\section{Introduction}
\label{sec:introduction}
\input{Content/0_introduction}

\section{Preliminaries}
\label{sec:preliminaries}
\input{Content/1_preliminaries}

\section{Proposed Solutions}
\label{sec:proposedapproach}
\input{Content/4_proposedapproach}

\section{Experiments}
\label{sec:experiments}

\input{Content/5_experiments}

\section{Conclusion}
\label{sec:conclusion}
\input{Content/6_conclusion}

\bibliographystyle{ACM-Reference-Format}
\bibliography{bibliography}
%


\end{document}

%% file: Content/0_introduction.tex
Quantum annealing (QA) is a quantum-based computational approach that leverages quantum phenomena such as entanglement and superposition to tackle complex optimization problems, in various domains such as machine learning (ML) \cite{Caldeira2019},\cite{Mott2017}, \cite{Pudenz2013},  bioinformatics \cite{Dillion2022}, \cite{Mulligan2019}, \cite{ngo2023qutie}, and networking \cite{Guillaume2022}, \cite{Ishizaki2019}. QA tackles optimization problems through a three-step process. First, the problem is encoded using a \emph{logical graph} that represents the quadratic unconstrained binary optimization (QUBO) formulation of a given problem. Next, the logical graph is embedded into a Quantum Processing Unit (QPU), which has its own graph representation known as a \emph{hardware graph}. The embedding aims to ensure that the logical graph can be obtained by contracting edges in the embedded subgraph of the hardware graph. This process is commonly referred to as \emph{minor embedding}. Finally, the QA process is iteratively executed on the embedded QPU to find the optimal solution which results in the minimum value for the given QUBO.

The minor embedding process is a major bottleneck that prevents QA from scaling. In particular, the topology of the hardware graph or its induced subgraphs may not perfectly align with that of the logical graph. Therefore, minor embedding typically requires a significant number of additional qubits and their connections to represent the logical graph on the hardware graph. Consequently, this situation can lead to two primary issues that can degrade the efficiency of QA. First, the need for a significant number of additional qubits to embed the logical graph can strain the available resources of the hardware graph, leading to scalability issues in terms of limiting hardware resources. Furthermore, when the size of the logical graph increases, the running time required to identify a feasible embedding increases exponentially. As a result, QA with a prolonged embedding time can be inefficient in practical optimization problems that require rapid decision-making.

There are two main approaches to address the above challenges in minor embedding, namely top-down and bottom-up approaches. The top-down approach aims to find embeddings of complete graphs \cite{Boothby2016}, \cite{Choi2011}, \cite{Klymko2012} in the hardware graph. Although embedding a complete graph can work as a solution for any incomplete graph with the same or smaller size, the embedding process for incomplete graphs, especially sparse graphs, usually requires much fewer qubits compared to embedding a complete graph. Specifically, experiments in\mbox{~\cite{Goodrich2018}} show that embedding techniques for sparse graphs can reduce the qubit usage by up to three times compared to the number of qubits required to embed complete graphs. As a result, the top-down approach may not be the most effective strategy for embedding sparse logical graphs. Although post-process techniques have been considered to mitigate this problem \cite{Goodrich2018}, \cite{Serra2022}, handling sparse graphs is a huge hindrance to the top-down approach.

On the other hand, the bottom-up approach directly constructs solutions based on the topology of the logical and the hardware graphs. This approach involves computing a minor embedding through Integer Programming (IP) which finds an exact solution from predefined constraints \cite{Bernal2020}, or progressive heuristic methods which gradually map individual nodes of the logical graph to the hardware graph \cite{Cai2014}, \cite{Ngo2023}, \cite{Pinilla2019}. Compared to top-down approaches, bottom-up methods are not constrained by the predefined embedding of complete graphs, allowing for more flexible embedding constructions. However, bottom-up approaches need to consider a set of complicated conditions associated with the minor embedding problem. As a result, the computational cost required to obtain a feasible solution is substantial.

To overcome the above shortcomings, we introduce a novel learning approach based on Reinforcement Learning (RL). Leveraging its learning capabilities, RL can rapidly generate solutions for previously unseen problem instances using its well-trained policy, thus efficiently enhancing the runtime complexity. Furthermore, through interactions with the environment and receiving rewards or penalties based on the quality of solutions obtained from the environment, RL agents can autonomously refine their strategies over time. This enables them to explore state-of-the-art solutions, which are particularly beneficial for emerging and evolving challenges such as minor embedding. Thus, RL is commonly applied to enhance the design of quantum systems~\cite{Khairy2020,Chen2022,Wauters2020,Fan2022,Ostaszewski2021,Saravanan_2024,LeCompte2023}.

Although RL cannot guarantee optimality, it has emerged as a powerful technique for \textit{handling} optimization problems~\mbox{\cite{BENGIO2021}}. However, applying it to our minor embedding problem has exposed several challenges. The first challenge is the stringent feasibility requirement for solutions in the minor embedding problem. Specifically, feasible solutions must fulfill a set of conditions for the minor embedding problem, including \mbox{\emph{chain connection}}, where each node in the logical graph $P$ is mapped to a connected subset of nodes (a "chain") in the hardware graph $H$; \mbox{\emph{global connection}}, ensuring that if two nodes in $P$ are connected by an edge, at least one corresponding edge must exist between their mapped chains in $H$; and \mbox{\emph{one-to-many}}, which requires that chains mapped from different nodes in $P$ do not overlap in $H$. All of these conditions are formally defined in Section~\mbox{\ref{sec:preliminaries}}. These constraints lead to an overwhelming number of infeasible solutions compared to the number of feasible solutions within the search space. As a result, it is challenging for the RL framework to explore a sufficient number of feasible solutions during the training phase as well as the testing phase.

The second challenge lies in the policy design of the RL framework. In particular, the policy of the RL framework is a mapping of observed input states to a choice of actions~\cite{Singh2000}. As recent advances in deep learning in representing high-dimensional inputs, deep learning models are utilized to parameterize the policy function~\cite{Mnih2015}. The deep learning model representing the policy (a.k.a. policy model) receives the information derived from the current state as input and returns the decision for the next action. Thus, the policy model is a key factor in the performance of the RL framework. In the context of applying RL to the minor embedding problem, the input state can include the topology information (of the given logical and hardware graphs) and the embedding information (which is a one-to-many mapping from the node set of the logical graph to the node set of the hardware graph). As the size of the two graphs increases, the number of state dimensions increases quadratically, potentially leading to an increase in the number of parameters required for the policy model to handle input states. On the other hand, reducing the state dimensions might degrade the policy model's performance due to a lack of valuable features. Therefore, the policy model in the RL framework for the minor embedding problem must be capable of extracting and combining valuable features from both the topology and embedding information in order to scale well with the size of the logical and hardware graphs without compromising performance. However, designing such a policy model is challenging because the topology representations and the embedding information are not identical, complicating the data combination process.

To realize an RL-based approach for the minor embedding problem, we first introduce an initial design, named \NAMEA, which sequentially embeds each node in a given logical graph into an associated hardware graph. Based on this initial design, we further introduce \NAMEB, incorporating three additional key components, to address the aforementioned challenges. The first component is a GCN-based architecture that integrates features of the logical graph, the hardware graph, and the embedding between the two graphs. This integration allows the architecture to return action probabilities and state values, serving as a model for the policy in the RL framework. The second component is a state transition algorithm that sequentially embeds a node in the logical graph into a chain of nodes in the hardware graph in order to ensure the validity of the resulting solutions. Lastly, the third component is an order exploration strategy that navigates the RL agent towards areas in the search space with good solutions, facilitating the training process.

Our proposed solutions are extensively evaluated against three state-of-the-art methods, including an OCT-based technique, a top-down approach~\cite{Goodrich2018}, and two bottom-up approaches Minorminer~\cite{Cai2014} and ATOM~\cite{Ngo2023}. The results illustrate that \mbox{\NAMEB\ } outperforms all state-of-the-art methods in terms of qubit usage for sparse logical graphs. For other graph types, while it does not outperform the OCT-based approach—which requires significant runtime for solution refinement—it remains the most efficient fast method, providing the best solutions within a short runtime. In addition, the computational efficiency of \NAMEB\  is comparable to that of the fastest method in terms of running time.

\textbf{Organization.} The rest of the paper is structured as follows. Section~\ref{sec:preliminaries} introduces the definition of minor embedding and an overview of RL. Our method and its theoretical analysis are described in Section~\ref{sec:proposedapproach}. Section~\ref{sec:experiments} presents our experimental results. Finally, Section~\ref{sec:conclusion} concludes the paper.

%% file: Content/1_preliminaries.tex
This section formally defines the minor embedding problem in quantum annealing (QA) and presents the preliminaries needed for our proposed solutions. It includes a brief overview of ATOM for the minor embedding problem based on a concept of topology adaption \cite{Ngo2023}, basic concepts of reinforcement learning (RL), and Advantage Actor Critic (A2C) \cite{Mnih2016}, which will be partially adopted in our solution.

\subsection{Minor Embedding Problem}
\begin{table}[t]
    \centering
\begin{tabular}{|l|l|}
      \hline
      Notion & Definition \\
     \hline
     $\mathcal{P}(S)$ & Power set of set $S$. \\
     $P = (V_P, E_P)$ & Logical graph $P$ with the set of nodes $V_P$ and the set of edges $E_P$. \\
     $X_P^{(a)}, X_P^{(f)}$ & The adjacency matrix and the node feature matrix of $P$. \\
     $H = (V_H, E_H)$ & Hardware graph $H$ with the set of nodes $V_H$ and the set of edges $E_H$. \\
     $X_H^{(a)}, X_H^{(f)}$ & The adjacency matrix and the node feature matrix of $H$.\\
     $G[S]$ & Induced subgraph of $G$ for set of nodes $S \subseteq V(G)$.\\
     $\phi^{(t)}: V_P \rightarrow \mathcal{P}(V_H)$ & The embedding from $P$ to $H$ at the step $t$. \\
     $|\phi^{(t)}|$ & Total size (the number of qubits) of $\phi^{(t)}$, calculated by $\sum_{v \in V_P} |\phi^{(t)}(v)| $.  \\
     $U^{(t)}_P$ & The set of nodes in $P$ that are embedded to $H$ at the step $t$. \\
     $U^{(t)}_H$ & The set of nodes in $H$ that are embedded by a node in $P$ at the step $t$. \\
     $O = (\Bar{a}_1, \dots, \Bar{a}_{|S|})$ & The embedding order including $|S|$ different precomputed actions. $\Bar{a}_i \in S $. \\ 
     & $O$ can be considered as a permutation of the set $S$. \\
     $\mathcal{E}^{(S)}$ & The space of embedding orders which are permutations of the set $S$. \\
     $\mathcal{E}^{(S)}_{O_A}$ & The space of embedding orders which are permutations of the set $S$ \\
     & with the prefix $O_A$. \\
     \hline
    \end{tabular}
    \caption{Common notions used in this paper}
\label{table:terminology}
\end{table}

We first present fundamental concepts and terminologies needed to understand how QA solves optimization problems. Next, we describe the minor embedding problem in QA. The common terminologies used throughout this paper are described in Table~\ref{table:terminology}.

Given an optimization problem with a set of binary variables $\mathbf{x} = \{x_1, x_2, ... , x_n\}$ and quadratic coefficients $Q_{ij}$, the problem can be represented in a quadratic unconstrained binary optimization (QUBO) form as follows:
\begin{equation}
    f(\mathbf{x}) = \sum_{i,j=1}^{n} Q_{ij}x_ix_j
\end{equation}
From this equation, we express the optimal solution for the given problem corresponding to the state with lowest energy of final Hamiltonian as $\mathbf{x}^{\ast}$ such that for $Q_{ij} \in \mathbb{R}$:

\begin{equation}
    \mathbf{x}^{\ast} = \textrm{arg min}_{\mathbf{x} \in \{0,1\}^n} \sum_{i,j=1}^{n} Q_{ij}x_ix_j
\end{equation}
\hspace{35mm} 
 
The QUBO formulation is then encoded using a logical graph. In the logical graph, each node corresponds to a binary variable $x_i$. For any two nodes $x_i$ and $x_j$, there exists an edge $(x_i, x_j)$ if the quadratic coefficient $Q_{ij}$ is non-zero.

The hardware graph is a representation of the Quantum Processing Unit (QPU) topology with nodes corresponding to qubits and edges corresponding to qubit couplers. The D-Wave QA system consists of three QPU topologies: Chimera (earliest), Pegasus (latest), and Zephyr (next generation). Each of these topologies is in a grid form: a grid of identical sets of nodes called unit cells. In this paper, we consider the Chimera topology because the embedding methods can be translated directly to the Pegasus topology without modification, since Chimera is a subgraph of Pegasus \cite{Boothby2020}. In other words, the method we develop for Chimera also works for Pegasus.

QUBO models a given optimization problem as a logical graph and the topology of a QPU used in QA as a hardware graph. In order to solve the final Hamiltonian that QA processes, we must find a mapping from the logical graph to the hardware graph. In the following, we formally define the minor embedding problem:

\begin{figure}[t]
\centering
	\includegraphics[width=.7\linewidth]{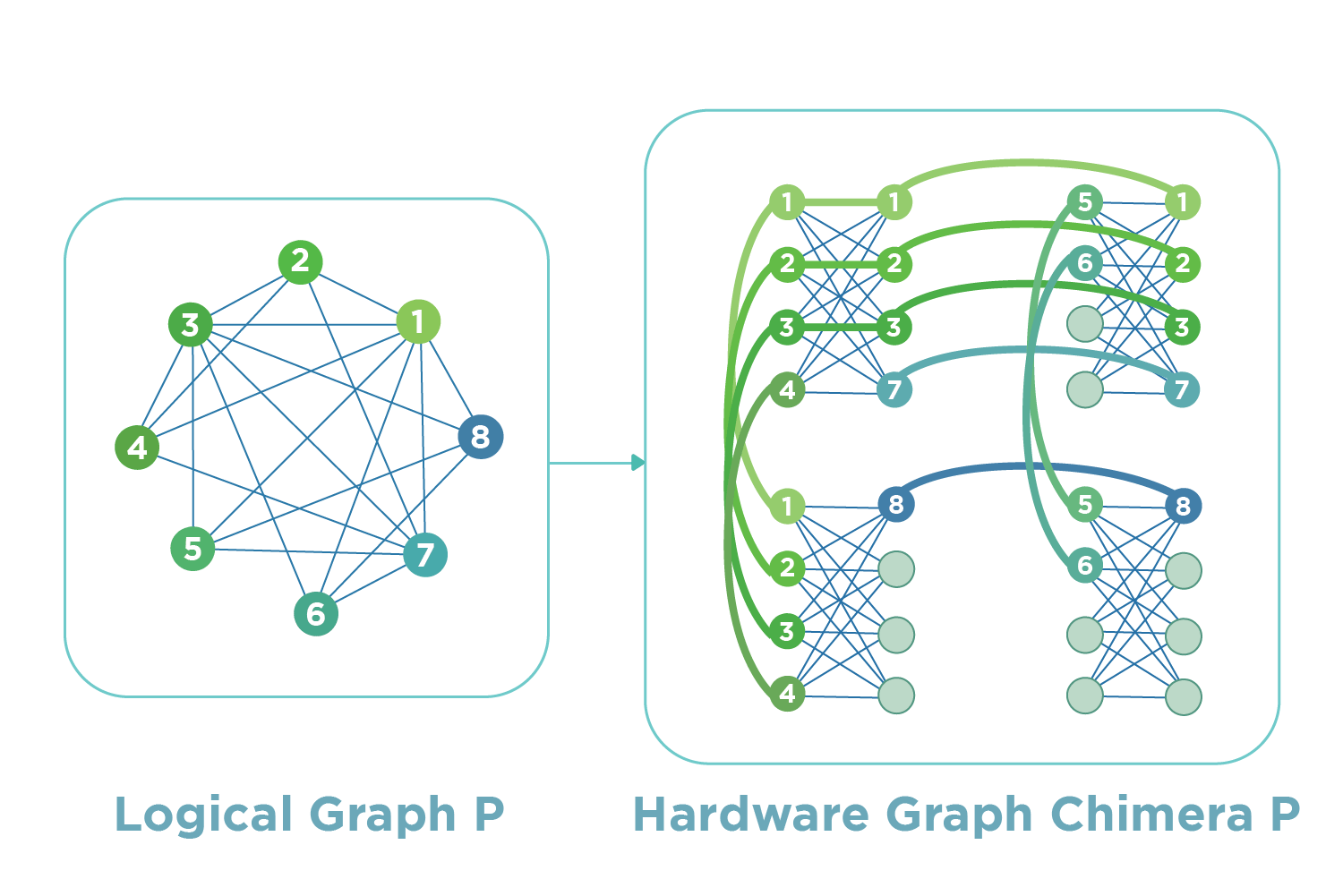}
	\caption{A high-level overview of the Minor Embedding problem where a QUBO formulation presented by a logical graph is embedded into a hardware graph.}
	\label{fig:problem1}
\end{figure}

\begin{definition}
Given a logical graph $P$ and a hardware graph $H$, the minor embedding problem seeks to find a mapping function $\phi: V_P \rightarrow \mathcal{P}(V_H)$ satisfying three embedding constraints:
\begin{enumerate}
    \item \textbf{Chain Connection}: We denote a subset of nodes in $H$ that are mapped from a node $u$ in $P$ as chain $\phi(u)$. $\forall u \in V_P$, any subgraph $H'$ of $H$ induced by a chain $\phi(u)$ from $H$ is connected.
    \item \textbf{Global connection}: For every edge $(u,v) \in E_P$, there exists at least one edge $(u', v') \in E_H$ such that $u' \in \phi(u)$ and $v' \in \phi(v)$.
    \item \textbf{One-to-Many}: Two chains $\phi(u)$ and $\phi(v)$ in the hardware graph do not have any common nodes $\forall u,v \in V_P, u \neq v$.
\end{enumerate}
\label{def:minor embedding problem}
\end{definition}
\noindent We note that an embedding which satisfies the above three embedding constraints is referred to as a \emph{feasible solution}. 

\subsection{Heuristic approaches for the minor embedding problem}

Heuristic methods employ a step-by-step approach to construct solutions. In the work of Cai et al.~\cite{Cai2014}, each step involves embedding a selected node from the logical graph into a suitable set of nodes in the hardware graph, while satisfying specific conditions. This iterative process continues until a feasible embedding is obtained. However, this method can encounter a problem when there is no appropriate set for the selected node. We call this an \emph{isolated problem}. To address this issue, the latest heuristic method, called \emph{ATOM}\cite{Ngo2023}, introduces the concept of \emph{adaptive topology}. Given a logical graph $P$ and a hardware graph $H$, this method consists of three main phases to find a feasible embedding from $P$ to $H$:
\begin{enumerate}
    \item \emph{Initialization:} A solution is initialized as $\phi \gets \emptyset$, and an embedded set is initialized as $U_P \gets \emptyset$. In addition, a permutation of the nodes in the logical graph, denoted by $O = (\Bar{a}_1, \dots, \Bar{a}_{|V_P|})$, is precomputed heuristically. This permutation is considered as a node order to guide the embedding process, where each node in the logical graph is iteratively embedded into a chain of nodes in the hardware graph.
    \item \emph{Path Construction:} In this phase, we iteratively embed nodes from $P$ into $H$ following the order of $O$ through $|V_P|$ steps. Specifically, in step $t$, the node $\Bar{a}_t$ is selected. Then, an algorithm, named \texttt{Node Embedding}, is designed to find a new embedding $\phi'$ such that the combination of $\phi'$ and $\phi$ is a feasible embedding for $P[U_P \cup \{\Bar{a}_t\}]$. If $\phi'$ exists, $\phi$ is updated with $\phi'$. This step is repeated for the subsequent $\Bar{a}_t$ in the order $O$ until $O$ is empty. Otherwise, in the cases where a feasible embedding $\phi'$ cannot be obtained for node $\Bar{a}_t$, it indicates the occurrence of an isolated problem. In such situations, phase 3 is triggered to handle this problem.
    \item \emph{Topology Adaption:} This phase is to handle the isolated problem. An algorithm, named \texttt{Topology Adapting}, is introduced to expand the current embedding $\phi$ to a new embedding $\phi_{ex}$ in such a way that the isolated problem no longer occurs. Then, it returns to phase 2 where $\phi$ is replaced by $\phi_{ex}$.
\end{enumerate}
The ATOM method guarantees to return a feasible solution after at most $3|V_P|$ iterations of phase 2. In addition, denoting the average degree of the logical graph as $d$, the \texttt{Node Embedding} and \texttt{Topology Adapting} algorithms have polynomial complexity, with $O(|V_H| \times d)$ and $O(|V_H|)$, respectively~\cite{Ngo2023}. Thus, ATOM has a total complexity of $O(|V_P|\times|V_H|\times d)$, making it efficient for hardware graphs in practical QA machines, which have qubit limits (i.e., $V_H$) ranging from $5000$ to $7000$~\cite{Vert2024}. However, one limitation of this approach is that the heuristic strategy used to determine the embedding order $O$ may lead to a high number of required qubits. With the use of RL, \NAMEB, is able to learn and provide better orders. Additionally, we incorporate \texttt{Node Embedding} and \texttt{Topology Adapting} as two subroutines in our method.

\subsection{Reinforcement Learning}

Reinforcement learning (RL) is a framework for addressing the problem of an RL-Agent learning to interact with an environment, formalized as a Markov Decision Process (MDP) $\mathcal{M}:=\langle\mathcal{S}, \mathcal{A}, \mathcal{T}, \mathcal{R}, \gamma\rangle$, where $\mathcal{S}$ is the set of possible states, $\mathcal{A}$ is the set of possible actions, $\mathcal{T}$ is the transition function, $\mathcal{R}(s, a)$ is the reward function, and $\gamma \in [0, 1]$ is a discount factor that determines the importance of future rewards. In this framework, the RL-Agent and the environment interact in discrete time steps. At time step $t$, the RL-Agent observes a state $s_t \in \mathcal{S}$ from the environment. Subsequently, the RL-Agent selects an action $a \in \mathcal{A}$ based on a policy $\pi:\mathcal{S} \rightarrow \mathbb{P}(\mathcal{A})$ which defines a probability distribution over actions given each state. Formally, the action is sampled as $a_t \sim \pi(\cdot \mid s_t)$. Upon executing action $a_t$, the RL agent receives a reward $r_t = \mathcal{R}(s_t, a_t)$ and transitions to a new state $s_{t+1}$ according to the transition function $ \mathcal{T}: \mathcal{S} \times \mathcal{A} \to \mathcal{S}$. The tuple of $(s_t, a_t, r_t, s_{t+1})$ is then stored in a roll-out buffer $\mathcal{D}$, which is used to update policy $\pi$ later. 

Next, we present the method for policy optimization. To begin, it is essential to define what is an optimal policy. A common metric for evaluating a policy is the \textit{state-value function}, which quantifies the expected cumulative reward when following policy $\pi$ from a given state.

Formally, given a policy $\pi$, the state-value function for state $s \in \mathcal{S}$ is defined as the expected sum of future discounted rewards:

\begin{align}
\mathcal{V}^{\pi}(s) = \mathbb{E}_{\pi} \left[ \sum_{t=0}^{\infty} \gamma^{t} r_{t} \mid s_0 = s \right],
\label{eqn:value_function}
\end{align}

\noindent where $\gamma \in (0,1]$ is the discount factor that controls the trade-off between immediate and future rewards.

The goal of RL is to find the \textit{optimal policy} $\pi^*$ that maximizes the long-term expected return for any initial state $s_0 \in \mathcal{S}$. This can be formulated as the following optimization problem:
    
\begin{align*}
\pi^* = \arg \max_{\pi} \mathcal{V}^{\pi}(s_0).
\end{align*}

Actor-Critic~\mbox{\cite{Mnih2016}} is a specific type of \textit{Policy Gradient} methods~\mbox{\cite{Sutton1999}} designed to find the optimal policy. In this method, the policy is represented by a parameterized learning model \textit{actor}, denoted as $\hat{\pi}_\theta$, where $\theta$ represents the trainable parameters. A key advantage of the Actor-Critic method over other policy gradient approaches is the inclusion of \textit{critic}, another parameterized learning model that estimates the state-value function $ \mathcal{V}^\pi$. Specifically, the critic is represented as $ \hat{\mathcal{V}}_\omega $, where $\omega$ denotes its trainable parameters. The critic helps reduce the variance of policy gradient estimates, leading to more stable learning.

The actor model is trained by minimizing the following loss:

\begin{align}
\mathcal{L}_{actor}(\theta) &=\mathbb{E}_{\hat{\pi}_\theta, a_t \sim \hat{\pi}_\theta(\cdot|s_t),s_{t+1} \sim \mathcal{T}(s_t,a_t)}\left[\log \hat{\pi}_\theta\left(a_t \mid s_t\right) (\mathcal{R}(s_t,a_t) + \gamma\hat{\mathcal{V}}_\omega(s_{t+1})- \hat{\mathcal{V}}_\omega(s_t))\right] \nonumber\\
&\approx \frac{1}{|\mathcal{D}|}\sum_{(s_t,a_t,r_t,s_{t+1}) \in \mathcal{D}}\left[\log \hat{\pi}_\theta\left(a_t \mid s_t\right) (r_t + \gamma\hat{\mathcal{V}}_\omega(s_{t+1})- \hat{\mathcal{V}}_\omega(s_t))\right]
  \label{eq:actor loss}
\end{align}

On the other hand, the critic model is trained by minimizing the following loss:

\begin{align}
\mathcal{L}_{critic}(\omega)&=\mathbb{E}_{\hat{\pi}_\theta, a_t \sim \hat{\pi}_\theta(\cdot|s_t),s_{t+1} \sim \mathcal{T}(s_t,a_t)} \left[(r_t + \gamma\hat{\mathcal{V}}_\omega(s_{t+1})- \hat{\mathcal{V}}_\omega(s_t))^2\right] \nonumber \\
& \approx \frac{1}{|\mathcal{D}|}\sum_{(s_t,a_t,r_t,s_{t+1}) \in \mathcal{D}} \left[(r_t + \gamma\hat{\mathcal{V}}_\omega(s_{t+1})- \hat{\mathcal{V}}_\omega(s_t))^2\right]
  \label{eq:critic loss}
\end{align}

The actor and critic are updated simultaneously. Given the update factors for the actor and critic as $\alpha_1$ and $\alpha_2$, the parameters $\theta$ of the actor model and $\omega$ of the critic model can be updated as:
\begin{align}
 \theta \leftarrow \theta + \alpha_1 \nabla_\theta \mathcal{L}_{actor}(\theta)
 \label{eq:actor update}
\end{align}

\begin{align}
 \omega \leftarrow \omega + \alpha_2 \nabla_\omega \mathcal{L}_{critic}(\omega)
 \label{eq:critic update}
\end{align}

%% file: Content/4_proposedapproach.tex
In this section, we present our baseline RL framework, NaiveRL, and our advanced solution, CHARME (Section \ref{sec:RLFrame}). The three key components of CHARME are introduced in Subsections 3.2-3.4 along with its theoretical analysis.

\subsection{RL Framework}\label{sec:RLFrame}
Here, we present \NAMEA\ as an initial RL framework for the minor embedding problem. In this framework, solutions are constructed by sequentially mapping one node from a given logical graph to one node in a hardware graph. This approach gives an initial intuition of how to formulate an optimization problem using an RL framework. We then identify the drawbacks of this framework that correspond to the challenges mentioned before. Accordingly, we introduce \NAMEB, a chain-based RL framework that effectively addresses the above drawbacks of \NAMEA. Unlike \NAMEA, \NAMEB\  constructs solutions by sequentially mapping one node from a given logical graph to a chain of nodes in a hardware graph. By doing this, \NAMEB\ is able to maintain the satisfaction of three embedding constraints after each step of embedding, thus ensuring the feasibility of the resulting solutions. 

\subsubsection{\textbf{\NAMEA\  - An Initial RL Framework}}
We consider a simple RL framework that constructs solutions by sequentially embedding one node in logical graph $P$ to one unembedded node in the hardware graph $H$. The algorithm terminates when the solution found satisfies three embedding constraints: chain connection, global connection, and one-to-many connection. Alternatively, it stops if there is no remaining unembedded node in $H$. We define the states, actions, transition function, and reward function in the RL framework as follows:
   
\noindent {\bf States:} A state $s$ is a combination of the logical graph $P$, the hardware graph $H$, and an embedding $\phi: V_P \rightarrow \mathcal{P}(V_H)$. Specifically, we denote the state at step $t$ as $s_t = \{P, H, \phi^{(t)}\}$.

\noindent {\bf Actions:} An action is in the form of $(u, v)$ with $u \in V_H$ and $v \in V_P$. Note that $v$ is an unembedded node that is not in the current embedding $\phi$. The action $(u,v)$ is equivalent to an embedding node $v$ in $P$ to an unembedded node $u$ in $H$.

\noindent {\bf Transition Function:} When an action $(u, v)$ is taken, a new solution $\phi'$ with $\phi'(v) = \phi(v) \cup \{u\}$ is obtained. The new state is a combination of the logical graph $P$, the hardware graph $H$, and the new solution $\phi'$.
    
\noindent {\bf Reward Function:} At a step $t$, the reward $r_{t}$ for selecting an action $(u,v)$ consists of two terms. The first term, named \emph{Verification} $\vartheta_{t}$, is determined based on the quality of the resulting solution after taking an action $(u,v)$. Specifically, if the solution $\phi^{(t)}$ satisfies the three embedding constraints, $\vartheta_{t
}$ is the number of actions taken to reach that state (which is equivalent to $t$). Otherwise, $\vartheta_{t}$ is set to 0.

In addition, to address the challenge of sparse rewards caused by the large search space, we introduce an additional component for the reward function, named \emph{Exploration} $\varepsilon_{t}$. This component is designed to encourage the RL framework to initially learn from existing solutions before independently exploring the search space. In particular, $\varepsilon_{t}$ is calculated by the difference between the action found in our \NAMEA\  framework at step $t$, denoted as $a_{t}$, and the action found by an existing step-by-step heuristic method~\cite{Ngo2023} at step $t$, denoted as $\Bar{a}_{t}$. The extent to which exploration contributes to the reward function is controlled by a variable $\sigma \in [0,1]$. In the initial stages of the training phase, we set the variable $\sigma$ to $0$. It implies that in the early steps of training, the policy of the RL framework mimics the heuristic strategy. As the policy is progressively updated with precomputed solutions, $\sigma$ gradually increases in proportion to the agent's success rate in finding solutions that satisfy the three embedding constraints. In sum, the reward function $r_{t}$ for a given state $s_{t}$ can be expressed as follows:
    \begin{equation}
    r_t = -\vartheta_{t} - \varepsilon_{t} = -\vartheta_{t} - \left[ 1 - \sigma \right]_+\left(a_{t}-\Bar{a}_{t}\right)^2
    \label{eq:simple:reward}
    \end{equation}

However, a major drawback of this naive approach is its inability to consistently produce feasible solutions that satisfy the constraints of the minor embedding problem. As demonstrated in the Experiment section, the feasibility rate remains low even for small logical graphs. This poses a significant bottleneck for the quantum annealing process, as finding a feasible embedding may require an unpredictable number of reruns. Therefore, an enhanced RL framework, that can ensure the feasibility of its solutions, is needed.

\subsubsection{\textbf{CHARME: A Chain-Based RL Solution for the Minor Embedding Problem}}

\begin{figure*}[h]
\centering
	\includegraphics[width=.8\linewidth]{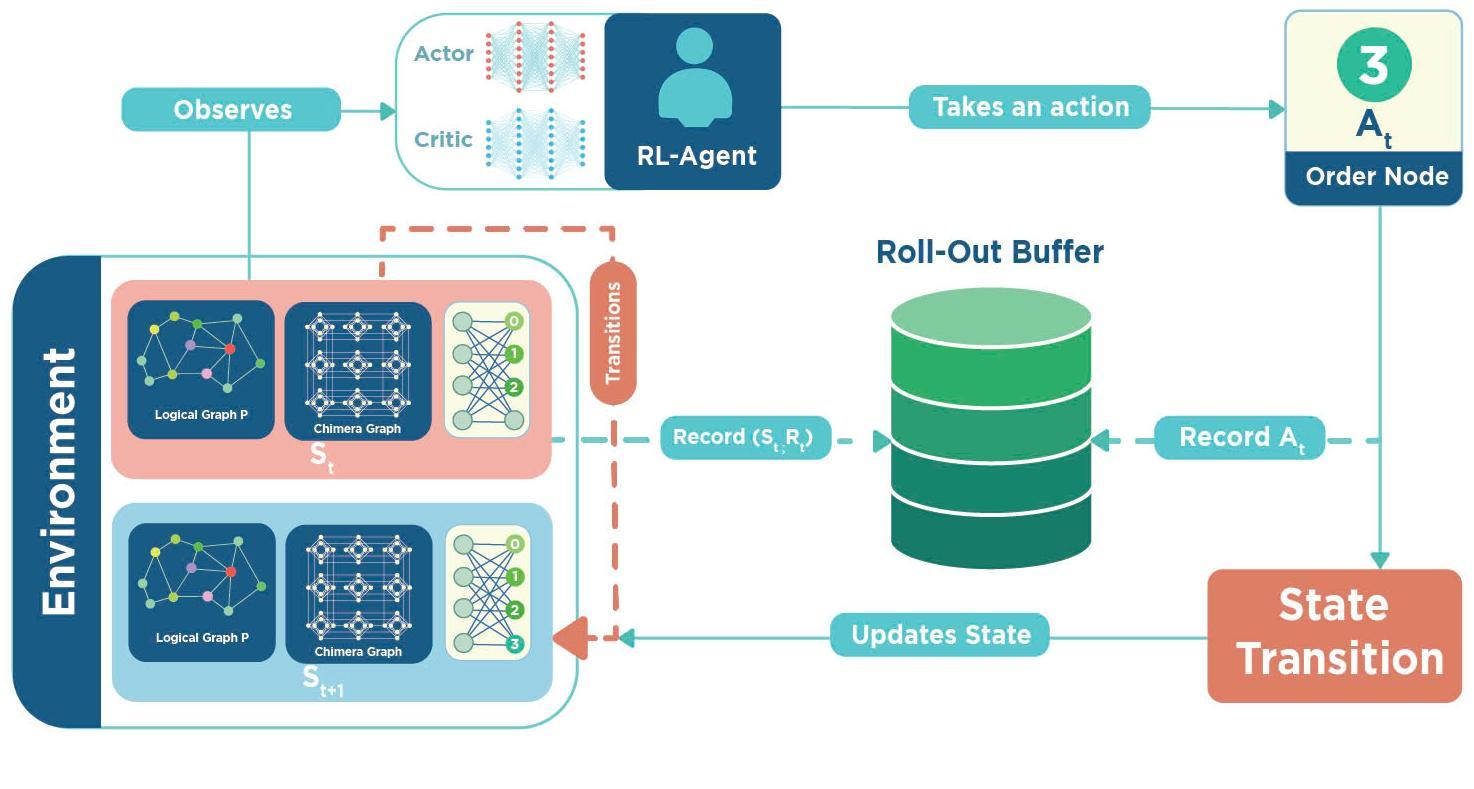}
	\caption{Workflow of the CHARME framework. The detailed architecture of the actor and critic is described in the subsequent part.}
	\label{fig:overview}
\end{figure*}

\begin{figure*}[h]
\centering
	\includegraphics[width=1\linewidth]{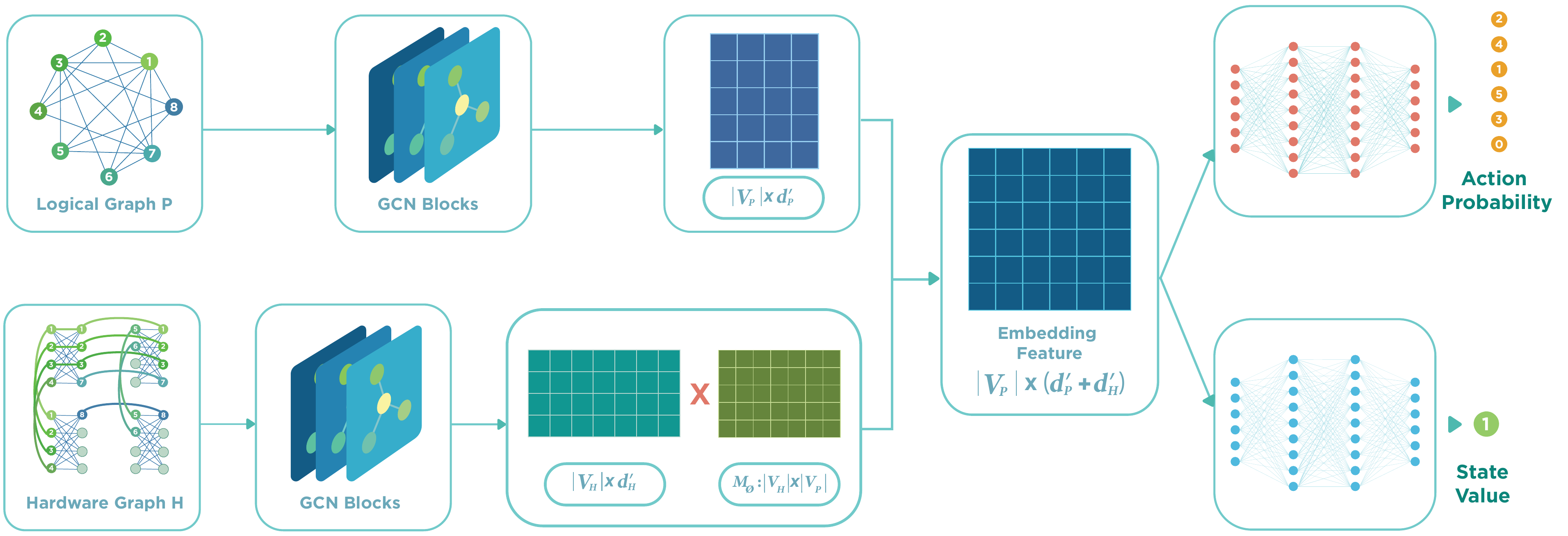}
	\caption{A GCN-based architecture of the models (actor and critic) presenting the policy of the RL-Agent}
	\label{fig:policy architecture}
\end{figure*}

To overcome the limitations of \NAMEA, we introduce \NAMEB, a chain-based RL framework that embeds one node in the logical graph into a chain of nodes in the hardware graph at each step. \NAMEB\  guarantees to find feasible solutions within exact $|V_P|$ turns of action selection and state transition (proven in Subsection~\ref{subsec:state_trans}). Consequently, compared to the naive approach, \NAMEB\ significantly reduces and simplifies the search space. In the following, we formally describe \NAMEB:
    
\noindent {\bf States:} We use the same definition of states as in the former framework. Here, we specify the components of states in more detail. The state $s_{t}$ at step $t$ contains:
    \begin{itemize}
        \item The logical graph $P$: We represent the logical graph $P$ as a tuple $(X_P^{(a)}, X_P^{(f)})$. $X_P^{(a)} \in \{0,1\}^{|V_P| \times |V_P|}$ indicates the adjacency matrix of $P$. Given the number of node features $d'_P$, $X_P^{(f)} \in \mathbb{R}^{|V_P| \times d'_P}$ indicates the node feature matrix of $P$. Subsequently, we specify the node features of $P$.
        \item The hardware graph $H$: We represent the logical graph $H$ as a tuple $(X_H^{(a)}, X_H^{(f)})$. $X_H^{(a)} \in \{0,1\}^{|V_H| \times |V_H|}$ indicates the adjacency matrix of $H$. Given the number of node features $d'_H$, $X_H^{(f)} \in \mathbb{R}^{|V_H| \times d'_H}$ indicates the node feature matrix of $H$. Subsequently, we specify the node features of $H$.
        \item The solution/embedding: We denote $\phi^{(t)}: V_P \rightarrow \mathcal{P}(V_H)$ as an embedding from $P$ to $H$ at step $t$. Specifically, if a node $v \in V_P$ is embedded in a node $u \in V_H$, then $u \in \phi(v)$.
    \end{itemize}
    
\noindent {\bf Actions:} At step $t$, we denote $U^{(t)}_P$ as a set of embedded nodes in $P$. In each step, an action $a_t$ is defined as the choice of an unembedded node $v \in V_P \setminus U^{(t)}_P$. From here, for simplicity, we consider each action as an unembedded node of $P$. Based on $a_t$, we can find the next feasible embedding $\phi^{(t+1)}$ of $P[U^{(t)}_P \cup \{a_t\}]$ in $H$.
    
\noindent {\bf Transition Function:} Because the given graphs $P$ and $H$ remain unchanged throughout an episode, we consider a transition from state $s^{(t)}$ to state $s^{(t+1)}$ as the transition from $\phi^{(t)}$ to $\phi^{(t+1)}$. The transition state $s^{(t)}$ to state $s^{(t+1)}$ is deterministic and does not depend on the policy. The algorithm for state transition (i.e., Algorithm~\mbox{\ref{alg:state transition}}) will be discussed in the subsequent section. In general, given an inputted state, Algorithm~\mbox{\ref{alg:state transition}} utilizes ATOM's subroutines~\mbox{\cite{Ngo2023}} (i.e., \texttt{Node Embedding} and \texttt{Topology Adapting} described in Section 2.2) to return a deterministic state as output.
   
\noindent {\bf Reward function:} At step $t$, a reward $r_t$ is determined to select a node $u \in P$ based on two components. The first component is the qubit gain compared to the current embedding, denoted as $c_t$. This component can be computed as the difference between the number of qubits of $\phi^{(t+1)}$ and $\phi^{(t)}$. The second component is the exploration term, which is similarly defined as that in the reward function of \NAMEA, with the difference being that the precomputed solution $\Bar{a}_t$ is obtained through our proposed exploration algorithm. This component is crucial in enhancing the training process of the \NAMEB\ framework. Specifically, at the start of the training process, the policy can be updated with high-quality embedding orders pre-computed by Algorithm~\ref{alg:order exploration}, rather than relying on randomly explored embedding orders. We recall that the action $a_t$ and the precomputed solution $\Bar{a}_t$ are presented as non-negative integers corresponding to the index of the selected node in the logical graph. Additionally, the exploration term is regulated by the variable $\sigma$, which progressively increases in proportion to the agent's success rate in finding feasible solutions. In sum, the reward $r_{t}$ is defined as follows:

\begin{equation}
    r_t = -c_t - \varepsilon_t = -(|\phi^{(t+1)}| - |\phi^{(t)}|) - \left[ 1 - \sigma \right]_+\left(a_{t}-\Bar{a}_{t}\right)^2
    \label{eq:hybridapproach:reward}
\end{equation}

In order to describe the workflow of \NAMEB, we denote a round of finding a solution for a pair of $(P, H)$ as an episode. Given a training set of logical and hardware graphs, \NAMEB\  runs multiple episodes to explore solutions of pairs of logical and hardware graphs from the training set. In one episode, the framework constructs a solution through a series of steps. Figure~\ref{fig:overview} shows the workflow of \NAMEB\ in an arbitrary step $t$. Specifically, at step $t$ of an episode, \NAMEB\ receives the state $s_{t}$ including the logical graph, the hardware graph, and the corresponding embedding, as input. Then, similar to \NAMEA, the actor-critic method is employed to represent and train the policy. Specifically, the actor maps the current state $s_t$ to an action $a_{t}$. The detailed architecture for the actor (and critic) is described in Section 3.2. Then, the state transition algorithm (Algorithm~\ref{alg:state transition}) takes the current state $s_{t}$ and the chosen action $a_{t}$ as input to produce the next state $s_{t+1}$. Finally, the environment updates the state $s_{t}$ to $s_{t+1}$, and calculates the reward $r_{t}$. The information in step $t$, including states, actions, and rewards, is stored in a roll-out buffer. The episode continues until the terminal state is reached where no unembedded node remains. After a certain number of training steps, all information in the roll-out buffer is used to update the actor and critic according to Equation~\ref{eq:actor update} and \ref{eq:critic update}. Then, a part of the roll-out buffer is randomly deleted to make space for information in subsequent episodes (i.e., $80\%$). In addition, $\sigma$ starts at $0$ and increases by $0.01$ after each update until $[1-\sigma]$ reaches $0$.

It is observable that for each episode the RL-Agent explores an embedding order $O = (a_1, \dots, a_{|V_P|})$ for a training pair $(P,H)$. By learning from reward signals along with embedding orders, \NAMEB\ is able to generate efficient embedding orders for unseen logical graphs and hardware graphs, consequently reducing the usage of qubits.

\subsection{Graph Representation and the Structure of the Policy}
In this section, we present an efficient representation of logical and hardware graphs and introduce a GCN-based structure for the models representing the policy.

As mentioned in the previous section, the logical graph $P$ at step $t$ is presented by a tuple $(X_P^{(a)}, X_P^{(f)})$. In this work, we use the constant feature matrix $X_P^{(f)} = [1]^{|V_P| \times 1}$ for all logical graphs $P$.
On the other hand, the hardware graph $H$ presented by a tuple $(X_H^{(a)}, X_H^{(f)})$. Specifically, the feature of a node $u$ in $H$ corresponds to the logical node embedded in $u$. If node $u$ is unembedded, the feature of $u$ is $-1$.

Based on the representation of $P$ and $H$, we propose a GCN-based structure for the models that represent the RL-Agent policy (Figure~\ref{fig:policy architecture}). Specifically, this structure is applied to the actor and the critic in the RL-Agent. The structure receives an input as a tuple of $(P, H, \phi)$ and returns the action probability of the actor and the state value of the critic. An important component of our proposed structure is a multilayer Graph Convolutional
Network (GCN)~\cite{Kipf2016} which is powerful for learning and analyzing graph-structured data~\cite{LeCun2015,Niepert2016}. In GCNs, the inputted graph is specified by an adjacency matrix $X^{(a)}$ and a feature matrix $X^{(f)}$. Those inputted matrices are propagated through a stack of $L$ layers represented by matrices $H^{(0)}, \dots, H^{(L)}$ by the layer-wise propagation rule as follows:
$$H^{(l+1)} = ReLU(\hat{D}^{\frac{-1}{2}}\hat{X}^{(a)}\hat{D}^{\frac{-1}{2}}H^{(l+1)}W^{(l)})$$  

Here, we have the adjacency matrix with added self-connections $\hat{X}^{(a)} = X^{(a)}+I$, the diagonal matrix of degrees $\hat{D}$ where $\hat{D}_{ii} = \sum_{j} \hat{X}^{(a)}_{ij}$, the trainable weight matrix of the l-\emph{th} layer $W^{(l)}$, hl{and the ReLU activation function.} The propagation rule is applied iteratively from $H^{(1)}$ to $H^{(L)}$ with the first layer $H^{(0)}$ initialized as $X^{(f)}$. For simplification, we denote this information passing mechanism as a non-linear function named $GCN$. Specifically, $GCN$ takes the adjacency matrix and the feature matrix as inputs and returns $H^{(L)}$ as output.

The flow of the structure starts with passing the representations of $P$ and $H$ through $GCN$. Subsequently, we obtain output matrices $Y_P \in \mathbb{R}^{|V_P| \times d'_P}$ and $Y_H \in \mathbb{R}^{|V_H| \times d'_H}$ respectively as follows:
\begin{equation}
    Y_P = GCN(X_P^{(a)}, X_P^{(f)})
\end{equation}
\begin{equation}
    Y_H = GCN(X_H^{(a)}, X_H^{(f)})
\end{equation}

We then group the nodes in $H$ that belong to the same chain in $\phi$ and aggregate their features. To do this, we present $\phi$ as a matrix $M_\phi \in \{0,1\}^{|V_H|\times |V_P|}$ such that for $\forall v \in V_P, u \in V_H$, $M_\phi[u,v] = 1$ if and only if $u \in \phi(v)$. As a result, we obtain the aggregated features $Y_{aggr} \in \mathbb{R}^{|V_P| \times d'_H}$ as follows:
\begin{align}
    Y_{aggr} = Y_H M_\phi
\end{align}
This step is important for several reasons. First, it eliminates unembedded nodes in the hardware graph. Information related to unembedded nodes is considered redundant because, given a fixed hardware topology, the position of unembedded nodes can be inferred from the embedded nodes. Hence, the information from the embedded nodes alone is sufficient to make a decision about the next action. As a result, filtering out unembedded nodes significantly reduces computational costs without hindering the convergence of the policy parameters. On the other hand, the information of nodes embedded in the same chain is aggregated. After aggregation, the output serves as the features of the chains in the hardware graph.

Next, we concatenate $Y_P$ and $Y_{aggr}$ to obtain the final embedding feature $Y_{final} \in \mathbb{R}^{|V_P| \times (d'_P+d'_H)}$. Subsequently, fully connected layers are applied to transform $Y_{final}$ into the desired output. The dimension of these layers varies depending on whether the desired output is the action probability or the state value.

We observe that the final embedding feature $Y_{final}$, which combines the information from the logical graph $P$, the hardware graph $H$, and the corresponding embedding $\phi$, has $O(|V_P|)$ dimensions. Thus, our proposed architecture addresses the challenge of state processing by efficiently combining the information from the graphs and the embedding.

\subsection{State Transition Algorithm}\label{subsec:state_trans}
\begin{algorithm}
\DontPrintSemicolon
  \KwInput{Current state $s_t = (P, H, \phi^{(t)})$, and an action $a_t$.}
  \KwOutput{The next state $s_{t+1}$}

    Let the set of current embedded node be $U^{(t)}_P\subseteq V_P$ which is inferred from $\phi^{(t)}$.
    
    $isolated := True$

    $\phi \leftarrow \phi^{(t)}$
    
    \While{$isolated$}
    {
        $\phi'$ := \texttt{Node Embedding}$(P, H, \phi, U^{(t)}_P, a)$
        
       \If{$\phi' \neq \emptyset$}{
            
            Initialize $\phi^{(t+1)}$ such that $\phi^{(t+1)}(v) := \phi(v) \cup \phi'(v)$ for $\forall v \in V_P$

            $X_{H}^{(f)}(u) := v \textrm{ with } \forall u \in V_H, v \in V_P, u \in \phi^{(t+1)}(v)$

            $isolated$ := False
       }
       \Else
       {   
            $\phi$:= \texttt{Topology Adapting}$(H, \phi)$
        }
    }
    \Return $s_{t+1} = (P, H, \phi^{(t+1)})$
\caption{\texttt{State Transition}}
\label{alg:state transition}
\end{algorithm}

We now present a state transition algorithm that computes the next state based on the current state and action. This algorithm integrates the \texttt{Node Embedding} and \texttt{Topology Adapting} algorithms introduced in \cite{Ngo2023} as subroutines.

Algorithm~\ref{alg:state transition} takes a current state $s_t$ at time step $t$ including a current logical graph $P$, a current hardware graph $H$, and current embedding $\phi^{(t)}$, as well as an action $a_t$ as input. It produces the next state $s_{t+1}$ including the logical graph $P$, the hardware graph $H$, and the next embedding $\phi^{(t+1)}$.

Before explaining this algorithm in detail, we recall the set of embedded nodes as $U^{(t)}_P$, which can be inferred from $\phi^{(t)}$. Algorithm~\ref{alg:state transition} must guarantee that the resulting next embedding $\phi^{(t+1)}$ is a feasible embedding of $P[U^{(t)}_P\cup \{a_t\}]$ to $H$. To realize this, we first initialize the decision variable $isolated$ as \emph{True} (line 2) and assign $\phi$ as $\phi^{(t)}$ (line 3). Then, the node $a_t$ is embedded into the current hardware graph $H$ by finding an additional embedding $\phi'$ through the \texttt{Node Embedding} algorithm. If the additional embedding $\phi'$ is not empty ($a_t$ is successfully embedded), Algorithm~\ref{alg:state transition} constructs a new embedding $\phi^{(t+1)}$ by combining $\phi$ and $\phi'$ (line 7). In addition, it updates new features for the hardware graph $H$ following a hardware representation discussed in the previous section (line 8). The next state $s_{t+1}$ is a combination of logical graph $P$, hardware graph $H$ with updated features, and new embedding $\phi^{(t+1)}$. In the end, we assign the value of $isolated$ to \emph{False}, finish the loop, and return the new state $s_{t+1}$. Otherwise, if $\phi' = \emptyset$, we cannot embed the node $a_t$ with the current embedding. We call this phenomenon the \textit{isolated problem}. In order to overcome this bottleneck, we expand the current embedding $\phi$ using the \texttt{Topology Adapting} algorithm (line 11) and repeat the whole process until we find a satisfied additional embedding $\phi'$.

From the lemmas provided in~\cite{Ngo2023}, it can be proven while assuming that the current embedding $\phi^{(t)}$ is feasible, the next embedding $\phi^{(t+1)}$ of the subsequent state $s_{t+1}$ returned by Algorithm~\ref{alg:state transition} is a feasible embedding of $P[U^{(t)}_P\cup \{a\}]$ in $H$. Additionally, the empty embedding $\phi^{(0)}$ initialized at step 0 is a feasible embedding of the empty logical graph $P[\emptyset]$ in $H$. Consequently, considering $T = |V_P|$, we can recursively infer that the final embedding $\phi^{(T)}$ corresponding to the terminal state $s_{T}$ obtained in the $T$\emph{-th} step is a feasible embedding of $P$ in $H$. Hence, with the state transition algorithm, \NAMEB\ effectively addresses the challenge of ensuring feasibility.

\subsection{Exploration Strategy}
\begin{algorithm}
\DontPrintSemicolon
  \KwInput{Set of generated logical graphs with $m$  elements $\mathcal{G}=\left(P_1, \ldots,P_m\right)$, the hardware graph $H = (V_H, E_H)$, the set of baseline orders $\mathcal{\Bar{O}} = \{\Bar{O}_1, \dots, \Bar{O}_m\}$, the sampling limit $D$ and the exploration limit $K$}
  \KwOutput{The embedding orders $\mathcal{O}=\left(O_1, \ldots,O_m\right)$}

  Initialize $\mathcal{O}=\left(O_1, \ldots,O_m\right)$ where each $O_j$ is a random permutation of the set of nodes $V_j$ with $j = \{1,2,\dots,m\}$
  
  Initialize potential score $\mu_j$ as $|V_H|$, for each graph $P_j$ with $j = \{1,2,\dots,m\}$, and store in $\mathcal{M}=\left\{\mu_1, \mu_2, \mu_3, \ldots, \mu_{m}\right\}$
  
  \While{$i\leq D$}{
    
    Sample a graph $P_i \in \mathcal{G}$ by following the selection probability $p(P_i|\mathcal{M})$ in Equation~\ref{eqn:graphprob}
    
    $\_, O', \_ :=\texttt{Order Refining}(P_i,H,0,\varnothing,\varnothing,K,F(O_i))$

    \If{$O' \neq \emptyset$ \textbf{ and } $F(O') < F(O_i)$}
    {
        $O_i \leftarrow O'$
        
        $\mu_i = F(O_i) \div F(\Bar{O}_i)$
    }

  }

  \Return $\mathcal{O}$
\caption{\texttt{Order Exploration}}
\label{alg:order exploration}
\end{algorithm}

A key challenge our proposed RL model faces is maintaining efficiency when training on large instances. As the number of nodes in the training logical graphs increases, the number of possible action sequences grows exponentially. Consequently, it becomes more challenging for the model to explore good action sequences for learning, potentially leading to an inefficient policy. In \NAMEB, we address this issue by incorporating an exploration term in the reward function, as defined in Equation~\mbox{\ref{eq:hybridapproach:reward}}. The intuition behind this exploration term is to guide the model to replicate precomputed actions $\Bar{a}_t$, during the initial training steps. As the policy "matures" through updates, the model gradually reduces its reliance on precomputed actions and begins exploring by itself. This approach allows the model to start from a region with high-quality solutions and efficiently explores from there, enhancing the overall efficiency of training.

Precomputed actions can be extracted from the solutions of the embedding method in~\cite{Ngo2023}. However, that method focuses on providing a fast minor embedding, so the sequence of precomputed actions is not carefully refined. In this section, we propose a technique for exploring efficient sequences of precomputed actions in training logical graphs, which makes the training process more efficient. To avoid confusion here, we mention the sequences that include precomputed actions as \emph{embedding orders}. We name our proposed technique as Order Exploration.

\subsubsection{Overview}
Given a logical graph $P = (V_P, E_P)$ with $T = |V_P|$ and a hardware graph $H = (V_H, E_H)$, an embedding order $O = (\Bar{a}_1, \dots, \Bar{a}_T)$ with $\Bar{a}_t \in V_P, t \in [1,T]$ results in a set of embeddings $\{\phi^{(1)}, \dots, \phi^{(T)}\}$. The position of an action in $O$ is the step in which the action is embedded in the hardware graph, so we denote an index of actions in $O$ as an embedding step. As mentioned in previous sections, given an embedding step $t$, the embedding $\phi^{(t)}$ is constructed by inputting Algorithm~\ref{alg:state transition} with $\phi^{(t-1)}$ and $\Bar{a}_{t-1}$ for $t \in [1,T-1]$. We note that the construction of $\phi^{(t)}$ may require the topology expansion (i.e. the function \texttt{Topology Adapting} is triggered). We call embeddings that require topology expansion as \textit{expansion embeddings}. Otherwise, we call them \textit{non-expansion embeddings}. The expansion and non-expansion notions are used in the remainder of this section. In addition, we denote $F(O) = |\phi^{(T)}|$ as the \emph{efficiency score} of $O$. \texttt{Order Exploration} aims to fast find embedding orders for training logical graphs such that their efficiency scores are as small as possible.

The overview of \texttt{Order Exploration} is described in Algorithm \ref{alg:order exploration}. Algorithm \ref{alg:order exploration} takes the inputs as the set of $m$ given training logical graphs, denoted as $\mathcal{G}=(P_1, \ldots, P_m)$, the set of baseline orders $\mathcal{\Bar{O}} = \{\Bar{O}_1, \dots, \Bar{O}_m\}$ (i.e., embedding orders of ATOM), the sampling limit $D$, and the exploration limit $K$. The training graph $P_i$ includes the set of nodes $V_i$ and the set of edges $E_i$. Given a logical graph $P_i \in \mathcal{G}$, we define the embedding order $O_i = (\Bar{a}_1,...,\Bar{a}_{|V_i|})$ as the sequence of precomputed actions for the training graph $P_i$. The goal of Algorithm \ref{alg:order exploration} is to quickly explore embedding orders for $m$ logical training graphs, denoted as $\mathcal{O}=(O_1, \ldots, O_m)$, such that their efficiency scores are as small as possible.

In general, Algorithm~\ref{alg:order exploration} operates by randomly selecting a graph $P_i \in \mathcal{G}$ and exploring its order. Specifically, the algorithm begins by initializing the embedding orders for logical graphs in the training set $\mathcal{G}$ (line 1). Given a graph $P_i = (V_i, E_i)$, $O_i$ is initialized as a random permutation of the set $V_i$. Then, each graph $P_i$ is assigned an equal potential score $\mu_i$ (line 2), which reflects the potential for further improvement of its embedding order. Each potential score is initially set to the worst-case ratio between two possible solutions, which is given by $|V_H|$.

Following initialization, the algorithm updates the embedding orders in $D$ steps. At each step, a graph $P_i$ is selected from the set $\mathcal{G}$ based on the distribution of potential scores (line 4). Specifically, the selection probability is determined as follows:

\begin{align}
p\left(P_i \mid \mathcal{M}\right)=\frac{\mu_i}{\sum_j \mu_j}
\label{eqn:graphprob}
\end{align}

Afterward, the order $O_i$ is refined by the subroutine \texttt{Order Refining} (line 5) which is detailed in the next section as Algorithm~\mbox{\ref{alg:branch and cut}}. If the refined order $O'$ is valid (i.e., $O' \neq \emptyset$) and achieves a better score than the current order $O_i$ (i.e., $F(O') < F(O_i)$), we re-assign $O_i$ as $O'$ (line 7), and update the potential score $\mu_i$ as follows:
\begin{equation}
\mu_i= \frac{F(O_i)}{F(\Bar{O}_i)}
\label{eqn:graphscore}
\end{equation}

We recall that $F(\Bar{O}_i)$ is the total number of qubits used for embedding $P_i$ resulting from ATOM~\cite{Ngo2023}. Thus, the score $\mu_i$, calculated by Equation~\ref{eqn:graphscore}, gives us the gap between our refined solution and the existing solution. Consequently, the larger the gap, the higher the likelihood that the corresponding graph will be selected for further refinement. It is also worth noting that if the initial potential scores are sufficiently large, every graph can be selected at least once. When the algorithm finishes discovering, it returns all feasible orders through the set $\mathcal{O}$ (line 9).

\subsubsection{Order Refining}
We observe that an embedding order is a permutation of $V_P$, so the space of embedding orders for $P$ is equivalent to the group of permutations of $V_P$. Therefore, finding embedding orders with a low efficiency score in this space is challenging. To address this issue, first, we establish an estimated lower bound on the efficiency scores for a special family of embedding orders. Then, based on the lower bound, we propose the subroutine \texttt{Order Refining} for fast exploration of orders, described in Algorithm~\ref{alg:branch and cut}.

Before going further, we introduce the notions used in this section. Given a finite set $S$, we define $\mathcal{E}^{(S)}$ as the group of permutations of $S$. From this, we infer that the space for embedding orders of $P$ is $\mathcal{E}^{(V_P)}$. We also refer to an order $O \in \mathcal{E}^{(V_P)}$ as a complete order. Considering two disjoint sets $S_1$ and $S_2$, we define the concatenation operator on the two groups $\mathcal{E}^{(S_1)}$ and $\mathcal{E}^{(S_2)}$ as the symbol $\oplus$. Specifically, given two orders $O_A = (\Bar{a}_1, \dots, \Bar{a}_{|S_1|}) \in \mathcal{E}^{(S_1)}$ and $O_B = (\Bar{b}_1, \dots, \Bar{b}_{|S_2|}) \in \mathcal{E}^{(S_2)}$, we have $O_A \oplus O_B = (\Bar{a}_1, \dots, \Bar{a}_{|S_1|}, \Bar{b}_1, \dots, \Bar{b}_{|S_2|}) \in \mathcal{E}^{(S_1 \cup S_2)}$. In addition, considering an order $O_A = \{\Bar{a}_1, \dots, \Bar{a}_{|S_1|}\} \in \mathcal{E}^{(S_1)}$, we denote the group of complete orders with the prefix $O_A$ as $\mathcal{E}^{(V_P)}_{O_A} \subseteq \mathcal{E}^{(V_P)}$. Specifically, an order $O \in \mathcal{E}^{(V_P)}_{O_A}$ is formed as $O = O_A \oplus O_B$ where $O_B \in \mathcal{E}^{(V_P \setminus O_A)}$.

Next, we consider a special prefix type, named a \textit{non-expansion prefix}. Given a subset $S_1 \in V_P$ with $t = |S_1| + 1$, and a prefix $O_A = \{\Bar{a}_1, \dots, \Bar{a}_{t-1}\} \in \mathcal{E}^{(S_1)}$, we say that $O_A$ is a non-expansion prefix if all orders $O = (\Bar{a}_1, \dots, \Bar{a}_{t-1}, \Bar{a}_{t}, \dots, \Bar{a}_{T}) \in \mathcal{E}^{(V_P)}_{O_A}$, result in a set of embeddings $\{\phi^{(1)}, \dots, \phi^{(T)}\}$ such that with $\forall t' \geq t$, $\phi^{(t')}$ is a non-expansion embedding.

Now, we consider constructing the embedding order by iteratively selecting an action as a node in $V_P$ and appending it to the current order. Each step in action selection is denoted as an embedding step. At the embedding step $t$, we assume that the embedding order $O_A = (\Bar{a}_1, \dots, \Bar{a}_{t-1})$ is a non-expansion prefix. In other words, the topology is no longer expanded from the embedding step $t$. Theorem~\ref{thm:bfexsubmodular} gives us the comparison between the number of qubits gained when selecting an action $v \in V_P$ at the embedding steps $t$ and $t' > t$ in this setting.

\begin{thm}
    Given a hardware graph $H = (V_H, E_H)$, a logical graph $P = (V_P,E_P)$, and the step limit $T = |V_P|$, assume that at the embedding step $t < T$, we have the sequence of selected actions $O_A = (\Bar{a}_1, \dots, \Bar{a}_{t-1})$ as a non-expansion prefix. Given a node $v\in V_P, v\notin O_A$, the number of qubits required to embed $v$ at step $t$ is not greater than that required when embedding $v$ at step $t'$ with $t < t' \leq T$.
    \label{thm:bfexsubmodular}
\end{thm}

\begin{proof}
     We denote the embedding at the embedding step $t-1$ as $\phi^{(t-1)}: V_P \rightarrow V_H$. Specifically, the embedding $\phi^{(t-1)}$ is constructed by embedding actions $\Bar{a}_1,\dots,\Bar{a}_{t-1}$ with each action as a node in the logical graph $P$. At the embedding step $t$, we denote the set of nodes in $V_P$ that are already embedded in the hardware as $U^{(t)}_P = \{v|v \in V_P, \phi^{(t-1)}(v) \neq \emptyset\}$, while we denote the set of nodes in $V_H$ that are embedded by a node in the logical graph as $U^{(t)}_H = \{u|u \in V_H, \exists v \in U^{(t)}_P, u \in \phi^{(t-1)}(v)\}$. We note that the set $U^{(t)}_P$ can be considered as an unordered representation of the sequence $O_A = \{\Bar{a}_1,\dots,\Bar{a}_{t-1}\}$ because it contains all elements of $O_A$. For mathematical convenience in this proof, we primarily use the set $U^{(t)}_P$ to represent the sequence $O_A$.
    
    We define a path of length $L$ in the hardware graph as $Z = \{z_1, \dots, z_{L}\}$ with $(z_i, z_{i+1}) \in E_H, i\in [1,L-1]$. We consider a path $Z= \{z_1, \dots, z_{L}\}$ to be a \emph{clean path} at time $t$ if $z_i \in V_H \setminus U^{(t)}_H$ with $i \in [1, L-1]$ and $z_i \in U^{(t)}_H$ with $i = L$. From that, we denote the set of all possible clean paths in the embedding step $t$ as $\mathcal{Z}^{(t)}$. Given a hardware node $u \in V_H$ and a logical node $v\in V_P$, we denote the set of all clean paths from $u$ to the chain $\phi^{(t-1)}(v)$ as $\mathcal{Z}^{(t)}_{uv} = \{Z = \{z_1,\dots,z_{L}\}|Z \in \mathcal{Z}^{(t)}, z_1 = u, z_L \in \phi^{(t-1)}(v)\}$. Note that $\mathcal{Z}^{(t)}_{uv} = \emptyset$ if $\phi^{(t-1)}(v) = \emptyset$. Then, we denote the shortest clean path at embedding step $t$ from $u$ to $\phi^{(t-1)}(v)$ as \begin{equation}
        \tilde{Z}^{(t)}_{uv} = \arg \min_{Z \in \mathcal{Z}^{(t)}_{uv}} |Z|
        \label{eq:shortest_clean_path}
    \end{equation}
    At embedding step $t$, given an unembedded logical node $v \in V_P \setminus U^{(t)}_P$, we denote the number of qubits gained for embedding $v$ into $H$ as $C(v|t)$. From Algorithm~\ref{alg:state transition}, we find an unembedded hardware node $u \in V_H \setminus U^{(t)}_H$ such that the size of the union of the shortest clean paths from $u$ to chains of neighbors of $v$ is minimized. Thus, denoting the set of neighbors of $v$ in the logical graph as $N(v)$, we can calculate $C(v|t)$ as follows:
 \begin{equation}
        C(v|t) = \min_{u' \in V_H \setminus U^{(t)}_H} \left |\bigcup_{v' \in N(v)} \tilde{Z}^{(t)}_{u'v'}\right |
        \label{eq:qubit_gain}
    \end{equation}

    Next, we consider embedding the node $v$ at embedding step $t' = t + \delta$. We assume that an arbitrary set of actions $\Delta = \{\Bar{a}_t,\dots,\Bar{a}_{t+\delta-1}\}$ is embedded in intermediate embedding steps such that $v \notin \Delta$. We update the set $U^{(t')}_P = U^{(t)}_P \cup \Delta$ and $U^{(t')}_H = U^{(t)}_H \cup \bigcup_{i \in [t, t+\Delta-1]} \phi^{(i)}(\Bar{a}_i)$ based on the assumption that there is no expansion embedding between the embedding step $t$ and $t'$. Then, we establish the relation between $C(v|t)$ and $C(v|t')$. First, we observe that $U^{(t)}_H \subset U^{(t')}_H$. Thus, we imply that $\mathcal{Z}^{(t)} \subset \mathcal{Z}^{(t')}$ followed by $\mathcal{Z}^{(t)}_{uv} \subseteq \mathcal{Z}^{(t')}_{uv}$ for $u\in V_H, v \in V_P$. As a result, from Equation~\ref{eq:shortest_clean_path}, we have $|\tilde{Z}^{(t)}_{uv}| \leq |\tilde{Z}^{(t')}_{uv}|$. Then, based on Equation~\ref{eq:qubit_gain}, given an unembedded node $v \in V_P \setminus U^{(t)}_P$, we have: 
    $$C(v|t) \leq C(v|t')$$
    Therefore, the theorem is proven.
\end{proof}

\begin{algorithm}
\DontPrintSemicolon
  \KwInput{The training logical graph $P = (V_P,E_P)$, hardware graph $H = (V_H, E_H)$, the embedding step $t$, the current embedding $\phi$, the prefix $O_A = (\Bar{a}_1, \dots \Bar{a}_{t-1})$, the number of recursive steps $k$, and the qubit threshold $\zeta$}
  \KwOutput{Return a message variable $msg$, the selected suffix $O_B = (\Bar{a}_t, \dots \Bar{a}_T)$, and the number of remaining recursive steps $k'$}

  Initialize $T \leftarrow |V_P|$

    \If{$t > T$}
  {
    \If{$|\phi| < \zeta$} {
        \Return $\textrm{True}$, $\varnothing$, $k-1$
    } \Else {
        \Return $\textrm{False}$, $\varnothing$, $k-1$
    }
  }

  $\Bar{F} \leftarrow |\phi|$

  $U^{(t)}_P \leftarrow \{\Bar{a}_1, \dots, \Bar{a}_{t-1}\}$
  
  \For{$v \in V_P \setminus U^{(t)}_P$}{
    $\phi' := \texttt{State Transition}((P, H, \phi), v)$

    $C(v|t) \leftarrow |\phi'| - |\phi^{(t-1)}|$
    
    $\Bar{F} \leftarrow \Bar{F} + C(v|t)$
  }

  \If{$ \Bar{F} > \zeta$}{
    \Return $\textrm{False}$, $\varnothing$, $k-1$
  } 

  $k' = k$

  \While{$k' > 0$} {
    Randomly select $v \in V_P \setminus U^{(t)}_P$

    $\phi' := \texttt{State Transition}((P, H, \phi), v)$

    $msg, O_B, k' \leftarrow \texttt{Order Refining} \left(P, H, t + 1, \phi', O_A \oplus (v), k' - 1, \zeta\right)$

    \If{$msg$ is True}{
        \Return $\textrm{True}$, $(v) \oplus O_B$, $k'$
    }
  }

  \Return $\textrm{False}$, $\varnothing$, $0$

\caption{\texttt{Order Refining}}
\label{alg:branch and cut}
\end{algorithm}

Based on Theorem \ref{thm:bfexsubmodular}, given a non-expansion prefix $O_A = (\Bar{a}_1, \dots, \Bar{a}_{t-1})$, we establish the lower bound $\Bar{F}$ on the efficiency score of all embedding orders $O = (\Bar{a}_1, \dots, \Bar{a}_{t-1}, \Bar{a}_t, \dots, \Bar{a}_T) \in \mathcal{E}^{(V_P)}_{O_A}$ as:

\begin{align}
F(O) &= \left|\phi^{(t-1)}\right|+\sum_{t' \in [t,T]} C(\Bar{a}_{t'} \mid t') \nonumber\\
&\geq\left|\phi^{(t-1)}\right|+\sum_{t' \in [t,T]} C(\Bar{a}_{t'} \mid t)\nonumber\\ &=\left|\phi^{(t-1)}\right|+\sum_{v \in V_P \setminus O_A} C(v \mid t) = \Bar{F} 
\label{eqn:appqubits}
\end{align}

Subsequently, in our experiments, we show that for the general prefix $O_A$, Theorem~\ref{thm:bfexsubmodular} is still valid for most pairs $t$ and $t'$, proving the usefulness of the lower bound $\Bar{F}$ in general.

We now introduce a subroutine to explore embedding orders, described in Algorithm~\ref{alg:branch and cut}. Algorithm~\ref{alg:branch and cut} takes a training logical graph $P = (V_P, E_P)$ with $T=|V_P|$, the hardware graph $H = (V_H, E_H)$, the current embedding step $t$, the current embedding $\phi$, the prefix order $O_A = \{\Bar{a}_1, \dots , \Bar{a}_{t-1} \}$, the number of recursive steps $k$, and the qubit threshold $\zeta$. It is important to note that $\phi$ is the embedding of nodes in the prefix $O_A$. The objective is to find a suffix $O_B$ such that the efficiency score of $O = O_A \oplus O_B$ is less than $\zeta$ within $k$ recursive steps. Specifically, Algorithm \ref{alg:branch and cut} returns a message variable $msg$, the suffix $O_B$, and the number of remaining recursive steps $k'$. The message variable $msg$ indicates whether a feasible suffix $O_B$ is found. If $msg$ is equal to $False$, $O_B$ is assigned as $\varnothing$.

Here, we provide a detailed explanation of Algorithm~\ref{alg:branch and cut}. First, we set $T$ as the total number of nodes in the logical graph (line 1). If there are no more nodes to consider ($t > T$), Algorithm~\ref{alg:branch and cut} compares the qubit usage of the current embedding $\phi$ with the threshold $\zeta$, returning the appropriate message variable $msg$ (lines 2 - 6). Then, the estimated lower bound $\Bar{F}$ is calculated based on Equation~\ref{eqn:appqubits} (lines 7 - 12). If $\Bar{F}$ exceeds the qubit threshold $\zeta$, Algorithm~\ref{alg:branch and cut} terminates the current recursion (lines 13 - 14). Subsequently, a random node $v$ is selected and the next recursion at the embedding step $t+1$ is initiated with the updated prefix $O_A \oplus (v)$ and the number of remaining recursive steps $k'$ (lines 16 - 21). The subroutine, triggered at embedding step $0$, can return a complete embedding order with an efficiency score less than $\zeta$.

%% file: Content/5_experiments.tex
This section presents our experimental evaluation with two focuses: 1) confirming the efficiency of our proposed exploration strategy, 2) showing the limitation of NaiveRL and 3) showing the superiority of \NAMEB\ by comparing its performance to three state-of-the-art methods: OCT-based~\cite{Goodrich2018},Minorminer~\cite{Cai2014}, and ATOM~\cite{Ngo2023} in terms of the number of qubits used and the total running time.

\subsection{Setup}
\textbf{Dataset.} We conduct our experiments on synthetic and real datasets. For the synthetic dataset, we generate logical graphs using the Barabási-Albert (BA) model. We then categorize the logical graphs of the synthetic dataset into two independent sets: a training set and a test set. The training set, denoted as $\mathcal{G}_{syn}$, includes $70$ BA graphs with $n = 150$ nodes and degree $d = 10$. The test set consists of $600$ new graphs that are generated by varying the number of nodes $n \in \{80, 100, 120, 140\}$ and the degree $d \in \{2, 5, 10\}$. Specifically, for each pair $(n,d)$, we generate $50$ instances. Thus, in total, we have $50\times 4 \times 3 = 600$ instances in the test set. If there is no further explanation about the hardware setting, the hardware graph is set in form of Chimera topology which is a $45 \times 45$ grid of $K_{4,4}$ unit cells~\cite{Boothby2016}. The experiment on the synthetic dataset aims to evaluate the performance of \NAMEB\ on logical graphs of the same distribution.

Furthermore, we also conducted the experiment on a real dataset generated based on the QUBO formulations of test instances of the \emph{Target Identification by Enzymes (TIE)} problem proposed in~\cite{ngo2023qutie}. In particular, different metabolic networks for the TIE problem result in different QUBO formulations. As we mentioned in previous sections, each specific QUBO formulation corresponds to a logical graph. Thus, we can generate various logical graphs by varying the metabolic networks for the TIE problem. In this experiment, we collect $185$ metabolic networks from the KEGG database \cite{Kanehisa2000}. Then, we generate the corresponding logical graphs using the QUBO formulations with these metabolic networks as input. In the collected graphs, the number of nodes ranges from $6$ to $165$, while the number of edges ranges from $22$ to $578$. We split the dataset into three groups based on the ratio between the number of nodes $n$, and the number of edges $m$: low density group ($m \leq 2n$), medium density group ($2n < m \leq 3n$), and high density group ($m > 3n$). We use $70$ graphs for training and preserve the remaining for testing. We denote the training set $\mathcal{G}_{real}$. This experiment aimed to illustrate the effectiveness of our hybrid RL model in real-world scenarios.

\textbf{Benchmark.} To demonstrate the effectiveness of our Order Exploration strategy, we establish a comparative greedy strategy referred to as Greedy Exploration. In Greedy Exploration, nodes in the logical graph are selected for embedding in a recursive manner at each step. The order of node selection for recursion is determined by the number of additional qubits required to embed each node. Additionally, recursion at a given step is terminated if the number of embedded qubits exceeds the best solution found up to that point. 

Then, we proceed to compare the performance of our end-to-end framework, \NAMEB, with three state-of-the-art methods for the minor embedding problem, as outlined below:
\begin{itemize}
    \item \textit{OCT-based~\cite{Goodrich2018}:} Following the top-down approach, it minimizes the number of qubits based on their proposed concept of the virtual hardware framework. In this experiment, we use the most efficient version of OCT-based, named \emph{Fast-OCT-reduce}. Although OCT-based is effective in reducing the number of qubits, it suffers from high computational complexity, leading to a long running time to find a solution. This can pose a bottleneck for the QA process, especially when the frequency of QA usage has increased rapidly.

    \item \textit{Minorminer~\cite{Cai2014}:} A heuristic approach developed commercially by D-Wave, this method finds solutions by consecutively adding nodes of the logical graph to a chain in the hardware graph without considering the isolated problem (a node in the hardware graph can be embedded by more than one node from the logical graph). Then, the solution with overlapping embedded nodes is refined by iteratively removing and adding better chains. The algorithm terminates when the time limit is reached or a feasible solution is achieved. Minorminer is an efficient approach for sparse logical graphs.

    \item \textit{ATOM~\cite{Ngo2023}:} This approach is a heuristic method that finds solutions based on the concept of adaptive topology. The strength of this method is the ability to find solutions in short running times without compromising their quality. In addition, ATOM uses a heuristic strategy to decide the order of embedding nodes. Thus, by comparing its performance to that of ATOM, we can assess the effectiveness of RL in leveraging the order selection strategy.

    \item \textit{ATOM (random):} This method is a variant of ATOM presented in \cite{Ngo2023}. The difference is that ATOM (random) uses a random strategy to decide the embedding order. Specifically, the next node of the logical graph selected to be embedded in the hardware graph is chosen randomly from the set of unembedded nodes of the logical graph.

    \item \textit{ATOM (degree):} This method is another variant of ATOM presented in \cite{Ngo2023}. The difference is that ATOM (degree) uses a degree-based strategy to decide the embedding order. Specifically, the next node of the logical graph selected to be embedded is the node with the highest degree from the set of unembedded nodes of the logical graph.
\end{itemize}

\textbf{Hyper-parameter Selection.} In the exploration phase, we apply Order Exploration, described in Algorithm~\ref{alg:order exploration}, to precompute embedding orders for each graph $P_i \in \mathcal{G}_{syn} \cup \mathcal{G}_{real}$. The constants $\beta_i$ for $i\in[1,m]$ are determined as the solutions for $P_i$ found by ATOM~\cite{Ngo2023}. Specifically, we set the sampling limit $D$ as $10000$, the exploration limit $K$ as $10^6$ ,and use up to $1024$ threads running in parallel to explore embedding orders.

For training, we set the number of $GCN$ layers and fully connected layers for both actor and critic models as $L=3$ and $C=3$, respectively. The learning rate for the actor model is $3\times10^{-4}$, while the learning rate for the critic model is $10^{-5}$. The number of episodes to update the policy is $100$, and the batch size for each update is $290$.

For inference, due to the stochasticity of ML models, we sample $10$ solutions for each test case, and report the solution with the least number of qubits. In addition, we use multiprocessing to sample solutions, so the actual running time of inference is equal to the time to sample one solution.

\begin{table*}[ht]
\centering
\resizebox{\textwidth}{!}{
\begin{tabular}{c|ccc|ccc|ccc|ccc}
\toprule
\( n \) &
  \multicolumn{3}{c|}{80} &
  \multicolumn{3}{c|}{100} &
  \multicolumn{3}{c|}{120} &
  \multicolumn{3}{c}{140}\\ \cmidrule{2-13}
d &
  \multicolumn{1}{c}{2} &
  \multicolumn{1}{c}{5} &
  10 &
  \multicolumn{1}{c}{2} &
  \multicolumn{1}{c}{5} &
  10 &
  \multicolumn{1}{c}{2} &
  \multicolumn{1}{c}{5} &
  10 &
  \multicolumn{1}{c}{2} &
  \multicolumn{1}{c}{5} &
  10 \\ \midrule
Early embedded gain &
  \multicolumn{1}{c}{5.36} &
  \multicolumn{1}{c}{14.09} &
  17.15 &
  \multicolumn{1}{c}{6.01} &
  \multicolumn{1}{c}{15.28} &
  23.19 &
  \multicolumn{1}{c}{6.28} &
  \multicolumn{1}{c}{17.75} &
  30.66 &
  \multicolumn{1}{c}{7.07} &
  \multicolumn{1}{c}{21.78} &
  36.85 \\ 
Late embedded gain &
  \multicolumn{1}{c}{8.94} &
  \multicolumn{1}{c}{18.73} &
  25.70 &
  \multicolumn{1}{c}{11.57} &
  \multicolumn{1}{c}{22.81} &
  34.62 &
  \multicolumn{1}{c}{12.20} &
  \multicolumn{1}{c}{29.14} &
  56.02 &
  \multicolumn{1}{c}{17.75} &
  \multicolumn{1}{c}{38.72} &
  67.02 \\ \bottomrule
\end{tabular}
}
\caption{Comparison of the number of qubits required to embed a randomly selected node into a hardware graph before (early embedded gain) and after (late embedded gain) the expansion of the topology.}\label{befor-after-expand-table}
\end{table*}

\begin{figure*}[t]
    \includegraphics[width=\linewidth]{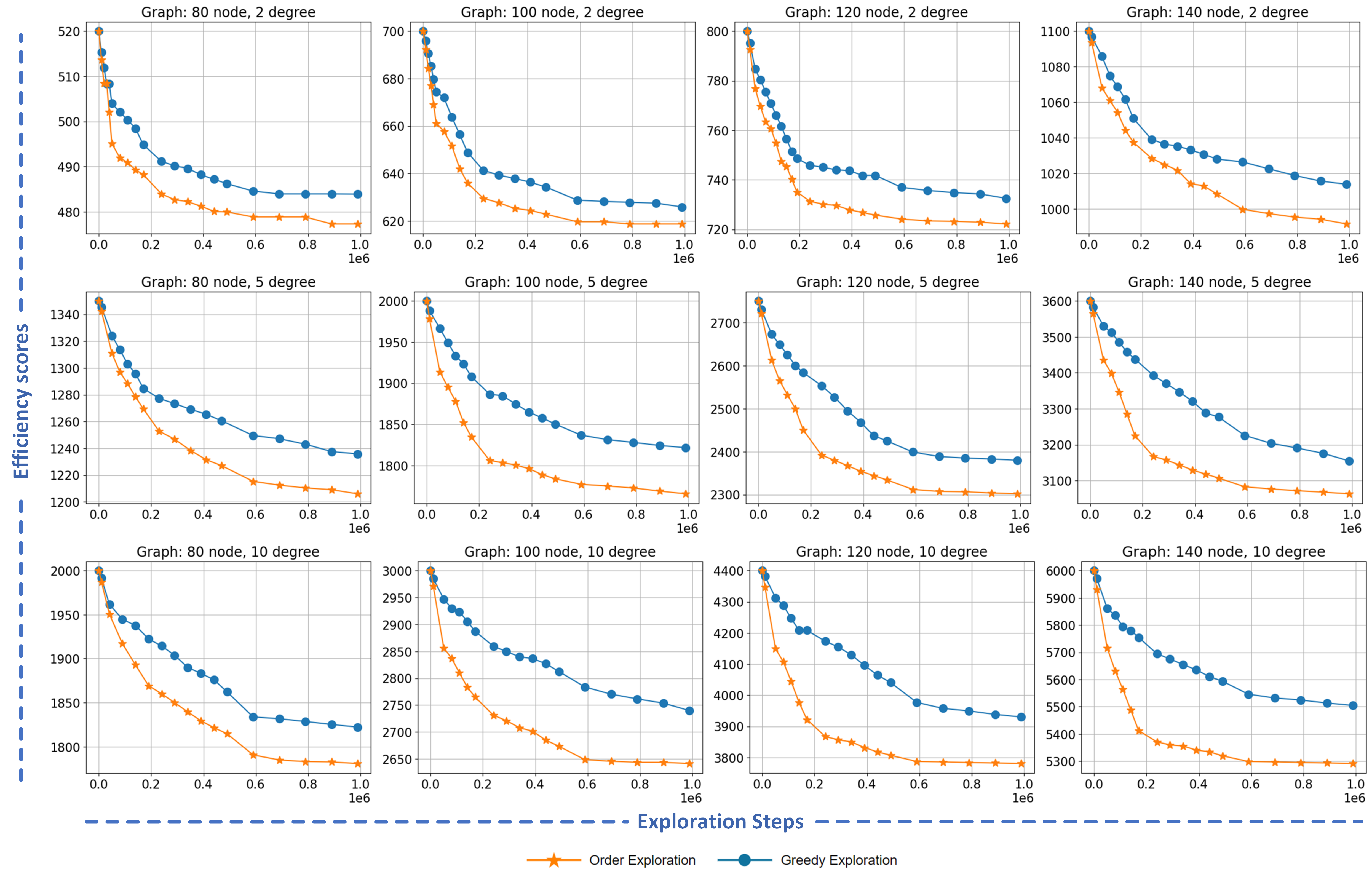}
    \caption{The figure highlights the performance of Order Exploration and Greedy Exploration in discovering embedding orders for training graphs in synthetic training sets under various settings of $(n,d)$. In subfigures, the x-axis denotes the number of exploration steps, while the y-axis represents the corresponding efficiency scores of the resulting orders. As a result, in these subfigures, each orange (blue) point indicates the best efficiency score of the embedding order resulted by Order Exploration (Greedy Exploration) after a specific number of exploration steps.}
    \label{fig:exploreR1}
\end{figure*}

\begin{figure*}[t]
    \includegraphics[width=\linewidth]{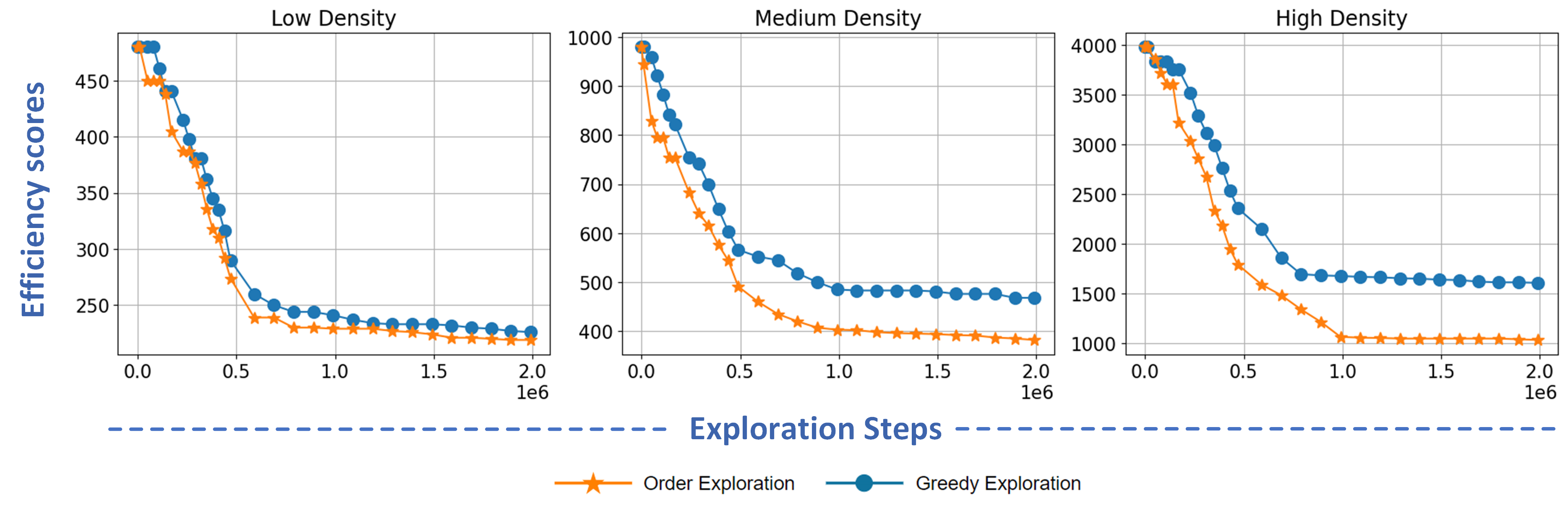}
    \caption{The figure highlights the performance of Order Exploration and Greedy Exploration in discovering embedding orders for training graphs in the real training sets under various graph types including low density, medium density, and high density. In subfigures, the x-axis denotes the number of exploration steps, while the y-axis represents the corresponding efficiency scores of the resulting orders. As a result, in these subfigures, each orange (blue) point indicates the best efficiency score of the embedding order resulted by Order Exploration (Greedy Exploration) after a specific number of exploration steps.}
    \label{fig:exploreR2}
\end{figure*}

\subsection{Results}
\subsubsection{The performance of the proposed exploration strategy} This section demonstrates the utility of Order Exploration in generating precomputed actions, commonly referred to as embedding orders, to streamline the training process of \NAMEB. Our experiments are designed to assess two key aspects: 1) the effectiveness of Theorem~\ref{thm:bfexsubmodular} with a general prefix that implies the practical significance of the lower bound outlined in Equation~\ref{eqn:appqubits}, and 2) the overall performance of our proposed exploration strategy in terms of the efficiency scores of the resulting embedding orders.

\noindent{\bf The Effectiveness of Theorem~\ref{thm:bfexsubmodular} with a General Prefix.} In this experiment, we designed a meticulous procedure to evaluate the effectiveness of Theorem~\ref{thm:bfexsubmodular} for general prefix orders, which are not necessarily non-expansion orders. Specifically, we compare the number of qubits gained when embedding the same node $v$ at two different embedding steps $t$ and $t+\delta$, denoted as $C(v|t)$ and $C(v|t+\delta)$. In Theorem~\ref{thm:bfexsubmodular}, we prove that $C(v|t) \leq C(v|t+\delta)$ if there is no expansion embedding between $t$ and $t+\delta$. Thus, our focus here lies on comparing $C(v|t)$ and $C(v|t+\delta)$ when there is at least one expansion embedding between these steps. Given a graph $P = (V_P, E_P) \in \mathcal{G}_{syn} \cup \mathcal{G}_{real}$, the procedure for selecting $v$, $t$, and $\delta$ unfolds as follows. First, we randomly select a subset $S_P \subset V_P$ by selecting a random prefix $O_A \in \mathcal{E}^{(S_P)}$. We select $v$ as a random node in $V_P \setminus S_P$. Next, we need to select $t$ and $\delta$ such that there exists an expansion embedding from $t$ to $t+\delta$. In detail, we start by iteratively embedding nodes following the prefix order $O_A$. Subsequently, we randomly embed nodes $v' \in V_P\setminus S_P,v' \neq v$ until we obtain an expansion embedding at step $|S_P|+\tau$. Finally, $t$ is assigned as $|S_P|+\tau - 1$ and $\delta$ is randomly selected such that $\delta \geq 1$. Performing this procedure repeatedly, we can obtain different $v$, $t$, and $\delta$ for each training graph. For each triplet $(v,t,\delta)$ of a given training graph, $C(v|t)$ and $C(v|t+\delta)$ is calculated similarly as discussed in Theorem~\ref{thm:bfexsubmodular}. For each setting of $(n,d)$, we measure the average of $C(v|t)$ and $C(v|t+\delta)$ on all triplet $(v,t,\delta)$ sampled from graphs with the setting $(n,d)$. We denote the average on $C(v|t)$ and $C(v|t+\delta)$ as \emph{early embedded gain} and \emph{late embedded gain}, respectively.

The results are illustrated in Table \ref{befor-after-expand-table}. Across all $(n,d)$ configurations, the average qubit requirement for embedding a node prior to topology expansion (early embedded gain) is consistently lower compared to post-topology expansion (late embedded gain). Interestingly, when the graph density increases, the difference in qubit gain between before and after topology expansion becomes more significant. This trend can be attributed to the expansion of the topology, which, on average, elongates the length of clean paths. Consequently, embedding a node after topology expansion tends to incur higher costs. These observations highlight the significance of the lower bound $\Bar{F}$ on the efficiency score of embedding orders with a general prefix, as derived from Equation~\ref{eqn:appqubits}.

\noindent{\bf The Efficiency in Exploring Orders.} In this experiment, we conducted a comparison between Order Exploration and Greedy Exploration to assess their effectiveness in exploring embedding orders for training logical graphs across various node sizes and degrees $(n, d)$. Figure~\ref{fig:exploreR1} presents subfigures corresponding to each setting $(n,d)$ of the logical graphs in the synthetic training set $\mathcal{G}_{syn}$. Each subfigure displays the efficiency scores of the embedding orders resulted from both strategies with variations in the exploration limit $K$ which is mentioned in Algorithm~\ref{alg:order exploration}. We recall that given a logical graph $P$ and a hardware graph $H$, the efficiency score of an embedding order $O$ is the number of qubits when embedding $P$ into $H$ following the order $O$.

The results depicted in the subfigures highlight the superior performance of Order Exploration over Greedy Exploration across all $(n,d)$ settings. In detail, the disparity becomes more significant with increasing values of $n$ and $d$. For example, with $(n,d) = (80,2)$, the average efficiency score of Order Exploration is approximately $2\%$ less than that of Greedy Exploration, while with $(n,d) = (140, 10)$, the gap increases to around $9\%$. Furthermore, we observe that the disparity also widens as the exploration limit $K$ increases from $0$ to $10^6$. 

The primary difference between the two methods lies in their recursion termination conditions. Specifically, Order Exploration utilizes a termination condition based on the lower bound on embedding non-expansion prefixes, as derived from Theorem~\ref{thm:bfexsubmodular}. In contrast, Greedy Exploration does not incorporate this lower bound in its termination condition. Consequently, Theorem~\ref{thm:bfexsubmodular} proves to be practically useful for guiding the exploration of embedding orders, particularly in scenarios involving large $(n,d)$ values and high exploration limits $K$.

Similarly, when applied to the real training set $\mathcal{G}_{real}$, the Order Exploration approach consistently outperforms Greedy Exploration in terms of efficiency scores (Figure \ref{fig:exploreR2}). Specifically, when dealing with low-density and medium-density graphs, the disparity between Order Exploration and Greedy Exploration is small. This is because the limited number of nodes and edges restricts the variation in qubit usage across different embedding orders. For example, in order to embed logical graphs with fewer nodes and edges, the number of required expansion steps, which consumes significantly more qubits than non-expansion steps, is small. Thus, the qubit usage between different embedding orders shows little variation. However, when dealing with high-density graphs, Order Exploration significantly outperforms Greedy Exploration. Specifically, Order Exploration requires approximately $1100$ qubits, while Greedy Exploration requires more than $1600$ qubits for its solution. This illustrates the superior performance of Order Exploration, particularly in settings involving high-density graphs.

\begin{figure*}[h]
\centering
\includegraphics[width=1\columnwidth]{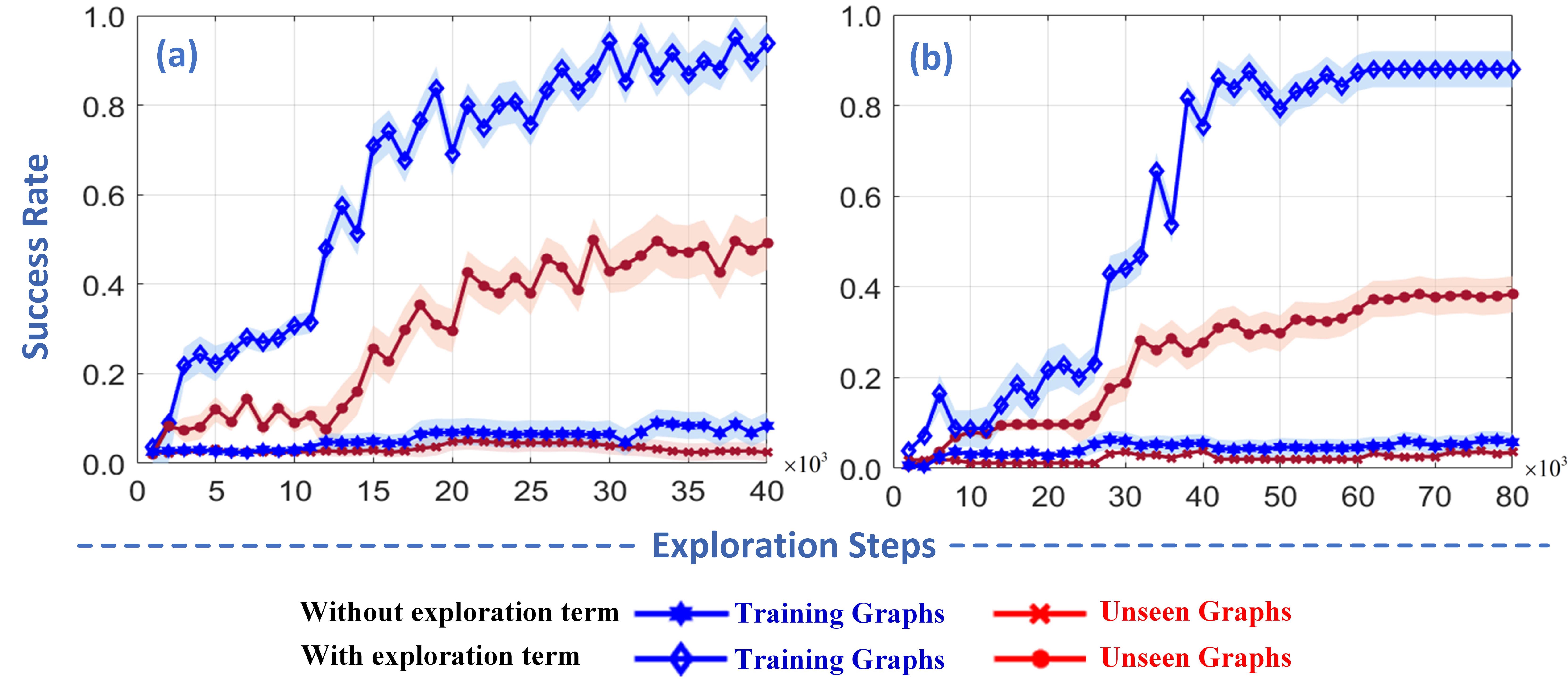}
\caption{The performance of a simple approach tested on (a) graphs with $30$ nodes and (b) graphs with $50$ nodes after $40,000$ and $80,000$ training steps, respectively. The $x$-axis represents the number of training steps, while the $y$-axis represents the success rate. The blue line indicates the training performance, while the red line indicates the test performance.}
	\label{fig:simple:learning performance}
\end{figure*}

\subsubsection{The limitation of \NAMEA}

Here, we evaluate the practical performance of \NAMEA\ to identify its limitation. In this experiment, we adopt a different setting. Specifically, we generate a dataset of $1000$ graphs and split them into training and test sets. Our goal is to train the model on $700$ BA graphs with $n = 70$ nodes and degree $d = 5$, aiming for effective generalization to the remaining $300$ graphs, which have varying sizes ($n \in \{30, 50\}$) and a lower degree ($d = 2$).

We employ the Actor-Critic method~\cite{Mnih2016} to update the policy, following Equations~\ref{eq:actor update} and \ref{eq:critic update}. Both the actor and critic are implemented as fully connected neural networks. During training, graphs from the training set are sequentially selected as initial states for exploration. The exploration history is then recorded and used to update the policy after all graphs in the training set have been processed.  

Figure~\ref{fig:simple:learning performance} highlights the limitations of \NAMEA\ in producing feasible solutions. Without the exploration term in Equation~\ref{eq:simple:reward}, \NAMEA\ struggles to generate feasible solutions during training, with a feasibility rate of less than $10\%$. Incorporating the exploration term improves feasibility by encouraging the model to mimic precomputed solutions, leading to a higher feasibility rate during training. However, this improvement does not generalize well to unseen graphs, where the feasibility rate drops to below $50\%$. Notably, \NAMEA\ encounters difficulties even with very small graphs. This suggests that its performance will deteriorate further as graphs become denser and larger. Thus, a naive RL approach is inadequate for addressing the minor embedding problem in quantum annealing, which fundamentally requires feasible solutions for every input graph.

To address these challenges, we propose \NAMEB, an enhanced RL framework designed to generalize effectively to unseen graphs while guaranteeing the feasibility of its solutions. In the following section, we demonstrate the effectiveness of \NAMEB\ through comprehensive evaluations.

\subsubsection{The Performance of \NAMEB}
In this experiment, we compare the performance of \NAMEB\ against three state-of-the-art methods on synthetic and real datasets.

\begin{figure}[t]
    \includegraphics[width=0.5\columnwidth]{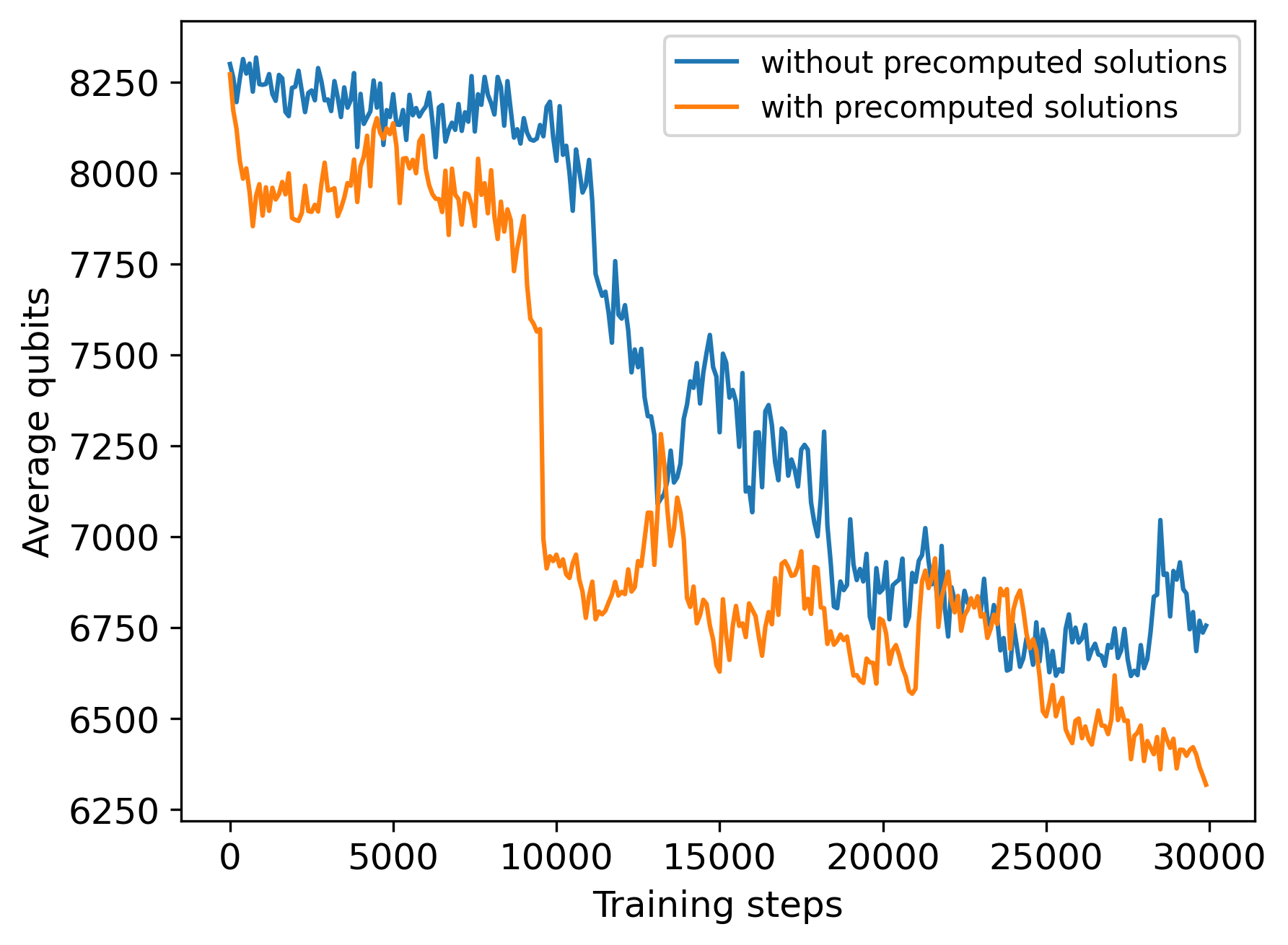}
    \caption{Training performance in term of average qubit usage of \NAMEB\ with and without precomputed solutions from Order Exploration.}
    \label{fig:training performance}
\end{figure}

\begin{table}[t]
\centering
\begin{adjustbox}{width=\columnwidth,center}
\begin{tabular}{@{}l|ccc|ccc|ccc|ccc@{}}
\toprule
n & \multicolumn{3}{c|}{80} & \multicolumn{3}{c|}{100} & \multicolumn{3}{c|}{120} & \multicolumn{3}{c}{140} \\ \midrule
degree & 2 & 5 & 10 & 2 & 5 & 10 & 2 & 5 & 10 & 2 & 5 & 10 \\ \midrule
OCT-based & 357.76 & \textbf{845.38} & \textbf{1222.11} & 515.74 & \textbf{1263.32} & \textbf{1850.9} & 725.84 & \textbf{1782.18} & \textbf{2607.06} & 931.24 & \textbf{2348.74} & \textbf{3488.18} \\
Minorminer & 338.56 & 1236.88 & 2110.16 & 501.56 & 1914.18 & 3313.24 & 664.76 & 2632.82 & 4854.10 & 830.56 & 3554.01 & 6573.84 \\
ATOM & 384.06 & 1126.14 & 1815.72 & 546.62 & 1782.86 & 2882.72 & 727.46 & 2560.36 & 4217.48 & 998.58 & 3448.24 & 5621.88 \\
ATOM (random) & 456.90 & 1411.14 & 2304.69 & 683.44 & 2146.00 & 3713.32 & 932.74 & 3089.38 & 5398.51 & 1234.72 & 4179.40 & 7476.58 \\
ATOM (degree) & 549.64 & 1482.06 & 2481.76 & 840.84 & 2242.20 & 3845.93 & 1096.64 & 3109.12 & 5530.90 & 1407.96 & 4103.83 & 7313.83 \\
\midrule
\NAMEB & \textbf{324.62} & 1027.36 & 1686.20 & \textbf{485.16} & 1580.08 & 2566.14 & \textbf{655.88} & 2263.66 & 3870.02 & \textbf{821.88} & 3082.59 & 5100.54 \\ \bottomrule
\end{tabular}
\end{adjustbox}
\caption{The average qubit usage of OCT-based, Minorminer, ATOM, and \NAMEB\  corresponding to 12 different graph sizes in the synthetic dataset.}
\label{table:synthetic qubits}
\end{table}

\begin{figure}[t]
    \includegraphics[width=0.5\columnwidth]{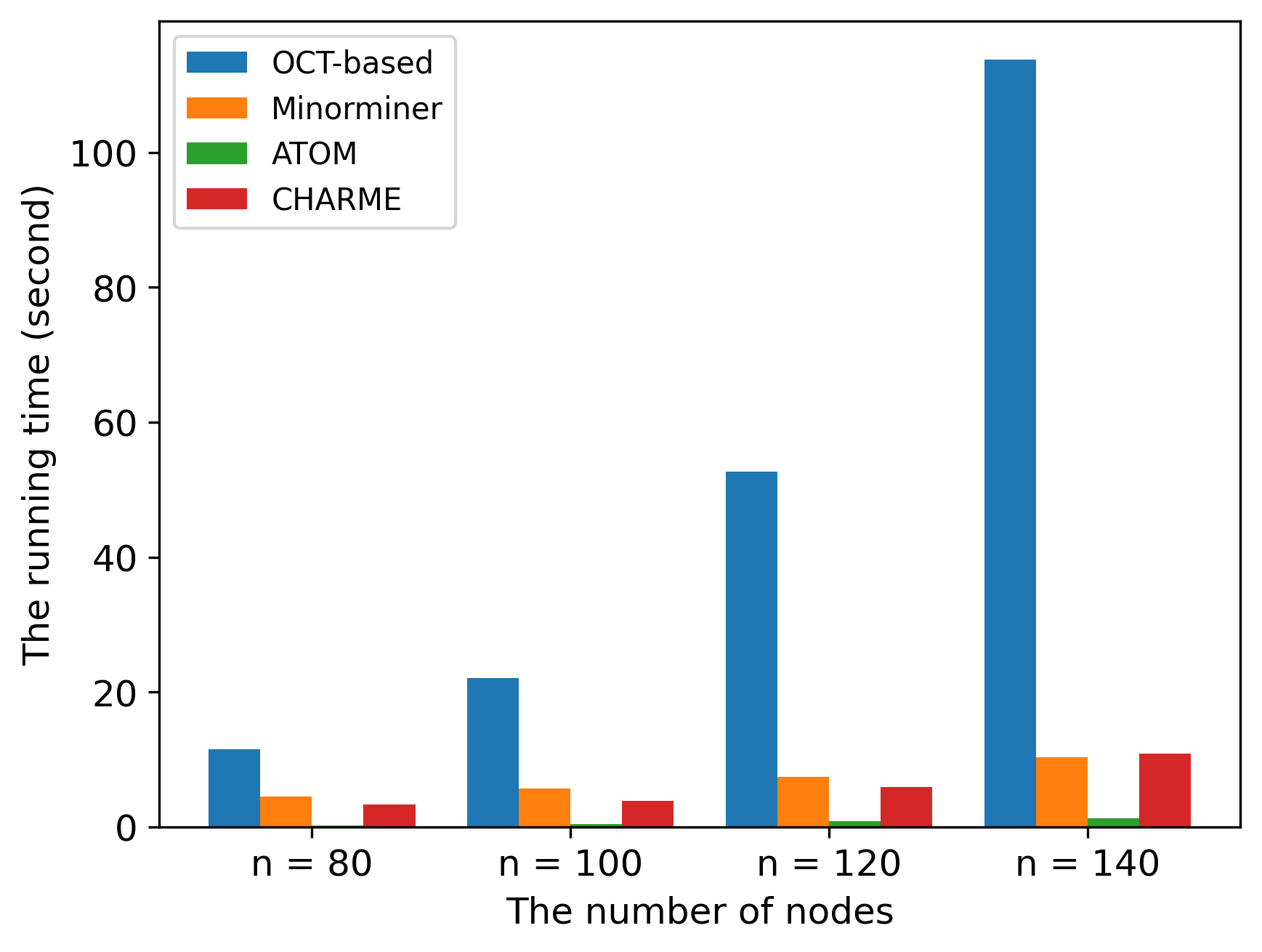}
    \caption{Comparison between four methods in terms of running time in the synthetic dataset.}
    \label{fig:synthetic running time}
\end{figure}

\begin{figure}[t]
    \includegraphics[width=0.5\columnwidth]{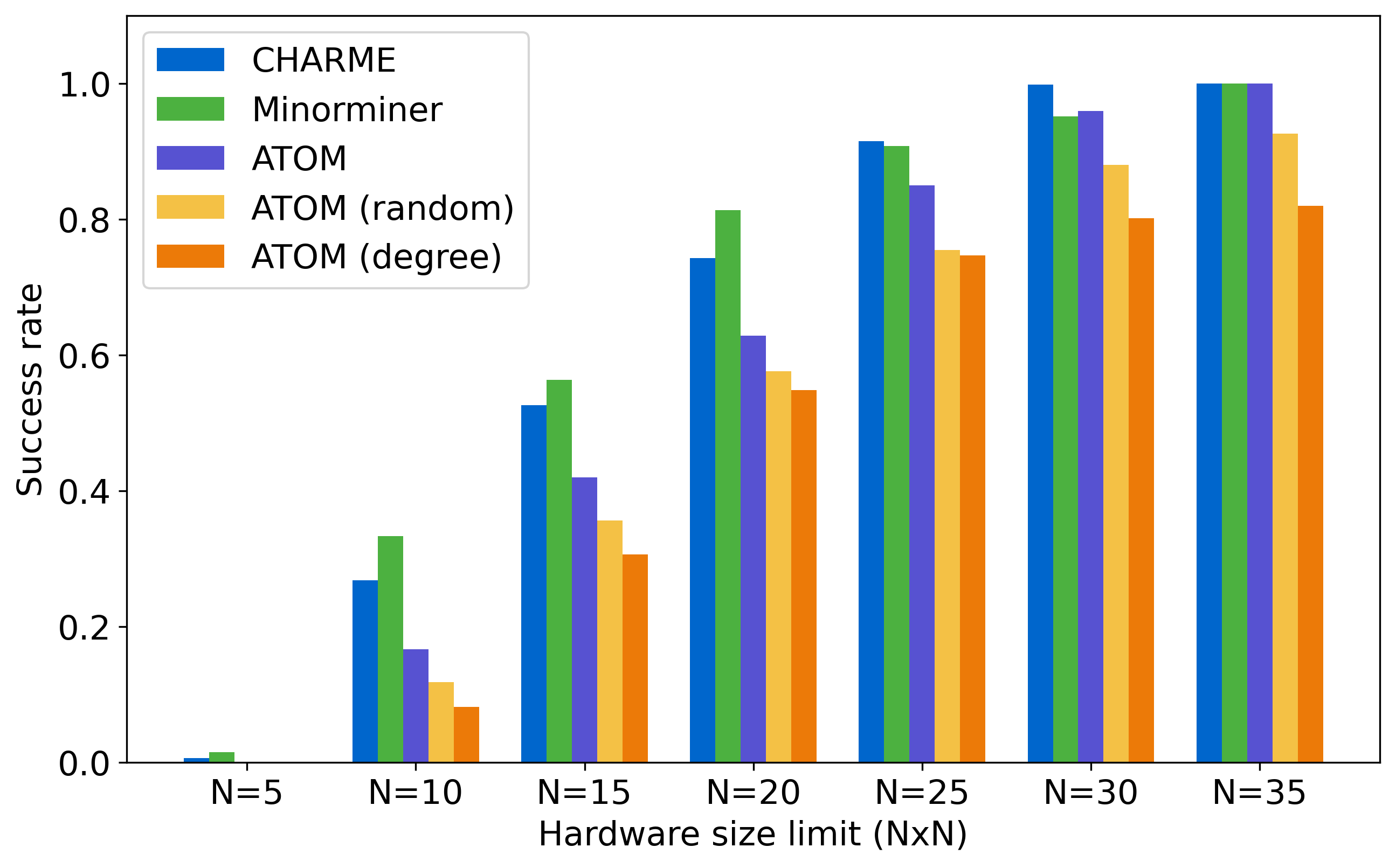}
    \caption{Comparison between four methods in terms of success rate with different hardware size limits. Hardware used in this experiment is in the form of $C(N,N,4)$ (an $N\times N$ grid of bipartite units $K_{4,4}$)}
    \label{fig:synthetic hardware limits}
\end{figure}

\noindent{\bf The Synthetic Dataset.}
Here, we assess the performance of \NAMEB\ on the synthetic dataset, focusing on the number of qubits used and the running time. The goals of this experiment are as follows:
\begin{itemize}
    \item We examine the training performance of \NAMEB\ with and without the precomputed solutions resulting from Order Exploration.
    \item We evaluate how effectively \NAMEB\ leverages the heuristic strategies to decide the order of the node embedding used in ATOM and its variants.
    \item We compare the quality of the solutions in terms of the number of qubits obtained by \NAMEB\ with three state-of-the-art methods.
    \item We measure the inference time of \NAMEB\ and compare it with the running time of other methods to assess its computational efficiency.
    \item We assess the success rate for finding embedding solutions in a short running time with different hardware size limits of \NAMEB\ and other methods.
\end{itemize}

Figure~\ref{fig:training performance} illustrates the training performance of \NAMEB\ in terms of average qubit usage in two scenarios: without and with precomputed solutions from Order Exploration. In the "without precomputed solutions" scenario, we remove the exploration term on the reward function. First, the results demonstrate that precomputed solutions accelerate \NAMEB's exploration for high-quality solutions. Specifically, without support from precomputed solutions, \NAMEB\ gets stuck on random solutions (i.e., average qubit usage $\approx 8200$) in the first $12000$ steps before eventually discovering good patterns (with average qubit usage dropping to $7200$). On the other hand, when using precomputed solutions, \NAMEB\ merges quickly to those solutions within the first $1000$ steps (i.e., average qubit usage $\approx 7800$). With a better starting point, \NAMEB\ explores high-quality patterns more rapidly, with the average qubit usage dropping to $6900$ by the $9000$\emph{th} step. In addition, we can see how efficient \NAMEB\ can explore in the solution space. In detail, after $30000$ steps of training, \NAMEB\ reduces the number of qubits required to embed graphs in the training set by approximately $25\%$.

 Table~\ref{table:synthetic qubits} presents the average number of qubits obtained from solutions generated by four methods for various pairs of size and degree $(n, d)$ of logical graphs in the test set. First, the results imply that \NAMEB\ leverages the heuristic strategies for the order of node embedding used in ATOM and its variants. Across all pairs $(n, d)$, \NAMEB\ consistently provides solutions with fewer qubits compared to ATOM. In particular, \NAMEB\ achieves reductions in qubit usage of up to $11.36\%$, $68.21\%$, and $72.78\%$ compared to ATOM, ATOM (random), and ATOM (degree), respectively. This observation indicates the advantage of using RL for determining the order of node embedding compared to heuristic and naive strategies. Additionally, the learning ability of the proposed RL framework is evident, as the policy learned from the training set performs well on new graphs of different sizes and degrees. 

Turning to solution quality, the results in Table~\ref{table:synthetic qubits} demonstrates that \NAMEB\ outperforms three state-of-the-art methods in terms of qubit usage in test cases with degree $d = 2$. On average, in these cases, \NAMEB\ achieves $2.52\%$ reduction in terms of qubit usage, compared to the best current methods. In other cases, although \NAMEB\ does not overcome the OCT-based approach in terms of qubit usage, it still outperforms Minorminer and ATOM. This observation, combined with the subsequent analysis of the running time, is important to illustrate the practical effectiveness of \NAMEB\ in handling the minor embedding problem.

Next, we examine the running time of OCT-based, Minorminer, ATOM, and \NAMEB. We note that the running time of ATOM, ATOM (random), and ATOM (degree) is identical, so we do not show the running time of ATOM (random) and ATOM (degree) in this experiment. The results in Figure~\ref{fig:synthetic running time} illustrate that the running time of OCT-based is significantly larger than the running time of three other methods. In particular, the running time of OCT-based is approximately 6 times bigger than the running time of the second slowest method, Minorminer. Thus, although OCT-based is efficient in terms of qubit usage in some certain cases, its expensive running time can pose a significant bottleneck in practical QA processes. On the other hand, the running time of \NAMEB\ is comparable to the fastest method, ATOM. Furthermore, as mentioned above, \NAMEB\ can offer better solutions than Minorminer and ATOM in all cases. Therefore, \NAMEB\ can be the optimal choice for tackling the minor embedding problem in practice, especially when current QA processes must handle user requests immediately due to increased demand for their use.

Finally, we assess the success rate in finding feasible solutions of fast embedding methods, including \NAMEB, Minorminer, ATOM, ATOM (random), and ATOM (degree) under different hardware size limits. The hardware used in this experiment is in the form of an $NxN$ grid of $K_{4,4}$ unit cells with $N \in \{5,10,15,20,25,30,35\}$. The running time limit in this experiment is set as $10$ seconds. The success rate is defined as the proportion of test cases (out of 600) in which a feasible solution is found within the running time limit. We note that Minorminer includes a post-processing function that aims to refine the final solution to fit the hardware size constraint, while \NAMEB, ATOM, ATOM (random), and ATOM (degree) do not include any function to refine solutions. The results in Figure~\ref{fig:synthetic hardware limits} illustrate that \NAMEB\ achieves a significantly higher success rate than ATOM, ATOM (random), and ATOM (degree) in all seven cases. In the case of $N=10$, the success rate of \NAMEB\ is twice that of ATOM and three times that of ATOM (random). This demonstrates the superiority of RL over heuristic and naive strategies in determining the order of node embedding by significantly boosting the success rate. Furthermore, \NAMEB's success rate is comparable to that of Minorminer (even achieves a higher success rate in two cases of $N=25$ and $N=30$). This observation indicates the efficiency and potential of \NAMEB\ in finding feasible solutions with limited hardware resources. Specifically, by determining a decent order of node embedding, \NAMEB\ achieves a success rate comparable to that of Minorminer without requiring refinement of the final solution.

\begin{figure*}[t]
\begin{subfigure}{.3\textwidth}
  \centering
  \includegraphics[width=1\textwidth]{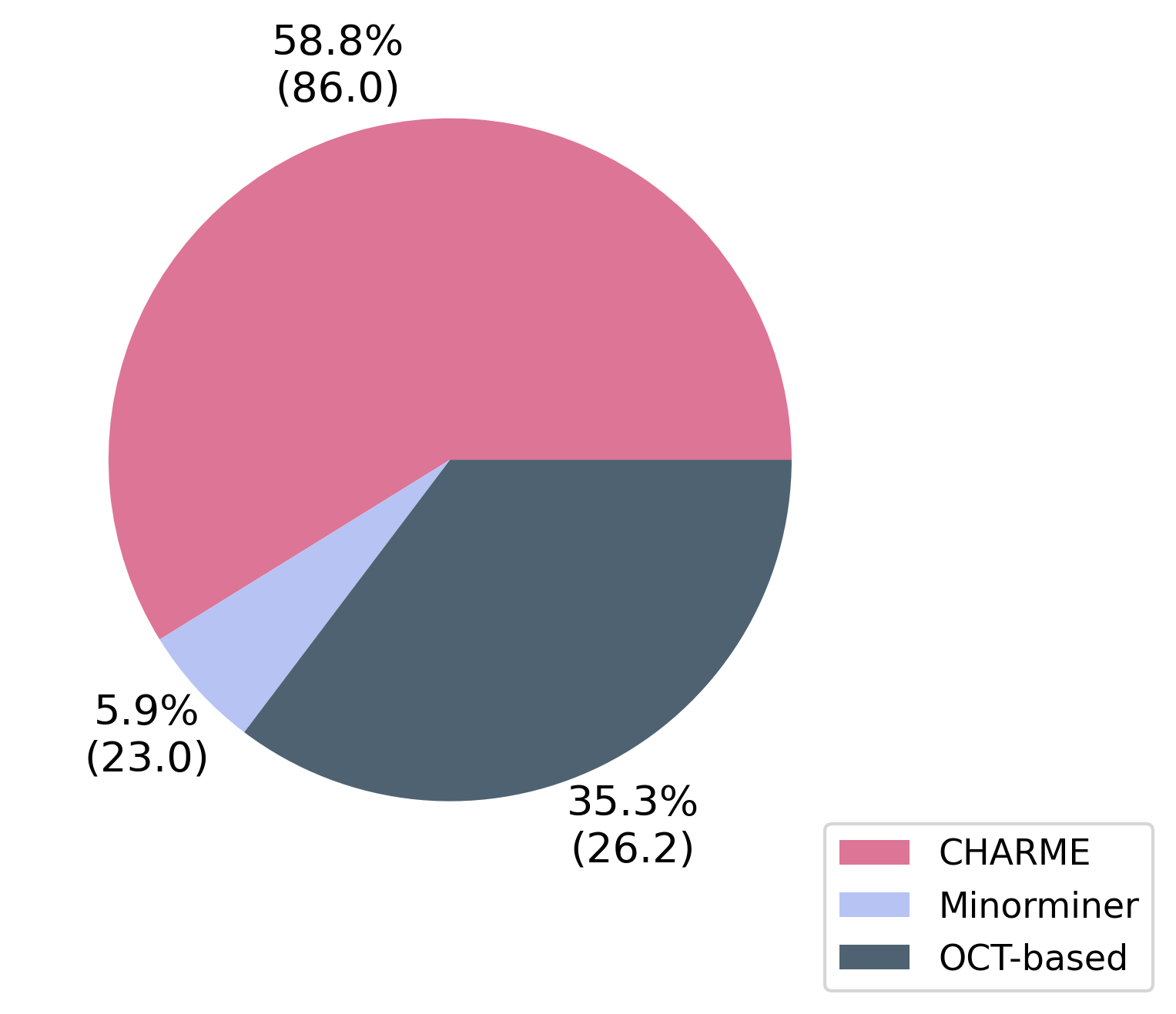}
  \caption{Low density}
  \label{fig:real world:low density}
\end{subfigure}
\hfill
\begin{subfigure}{0.3\textwidth}
  \centering
  \includegraphics[width=1\textwidth]{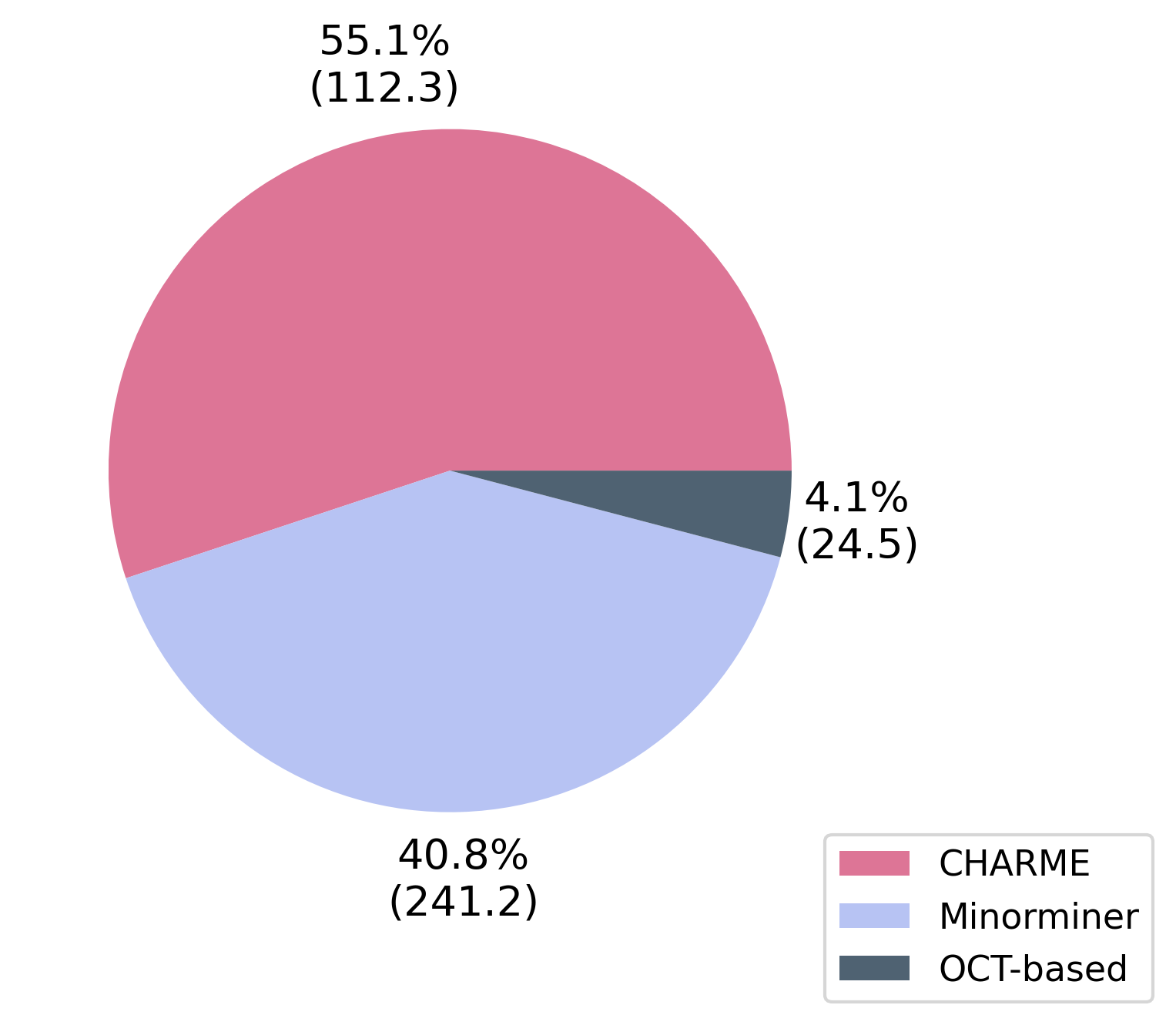}
  \caption{Medium density}
  \label{fig:real world:medium density}
\end{subfigure}
\hfill
\begin{subfigure}{.3\textwidth}
  \centering
  \includegraphics[width=1\textwidth]{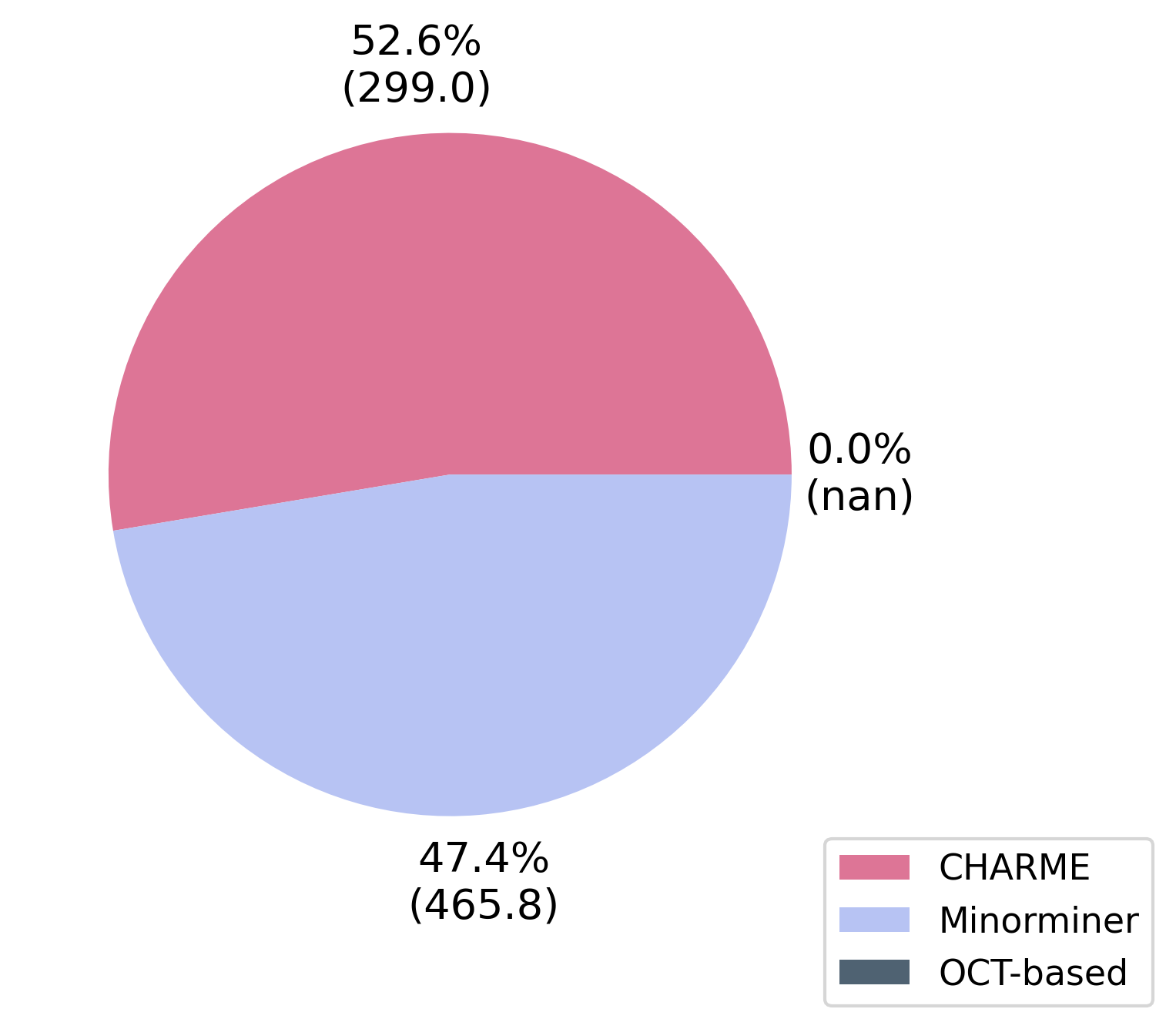}
  \caption{High density}
  \label{fig:real world:high density}
\end{subfigure}

\caption{The frequency of finding the best solutions and the corresponding average qubit usage (shown as numbers in parentheses)} of three methods applied to three groups of the real world dataset categorized as (a) low-density (b) medium-density (c) high-density.
\label{fig:real world:frequency}
\end{figure*}

\noindent{\bf The Real Dataset.} Here, we assess the performance of \NAMEB\ in comparison to OCT-based and Minorminer on a real dataset. To evaluate these methods, we measure the frequency at which each approach achieves the best solutions across three testing groups: low-density, medium-density, and high-density. Additionally, we analyze the running time of the three methods for various sizes of graphs. We note that in this experiment, ATOM, ATOM (random), and ATOM (degree) do not achieve the best solution in any case. Thus, in order to avoid redundancy, we do not include these methods in this experiment.

Figure~\ref{fig:real world:frequency} illustrates the frequency of finding the best solutions in terms of qubit usage of three methods for various testing groups including low-density, medium-density, and high-density. First, the \NAMEB\ method consistently outperforms competitors, achieving the highest frequency of finding solutions with the fewest qubits in the three testing groups. Remarkably, the frequency of \NAMEB\ is always above 50\%, illustrating its superiority. Another interesting observation is in the performance of the OCT-based method, which works well on the synthetic dataset, but struggles on the actual QUBO-based dataset. This phenomenon is due to the sparse nature of logical graphs formed by actual QUBO formulations, making bottom-up approaches like Minorminer and \NAMEB\ more efficient choices. In addition, the average qubit usage for the best solutions found by CHARME is higher than that found by Minorminer for low-density graphs. This suggests that CHARME performs better on low-density graphs with a large number of nodes, whereas Minorminer is more effective for graphs of a smaller size. However, this trend reverses for medium and high-density graphs, where CHARME excels on smaller graphs, and Minorminer proves to be more effective on larger ones.

\begin{figure}[t]
    \includegraphics[width=0.5\columnwidth]{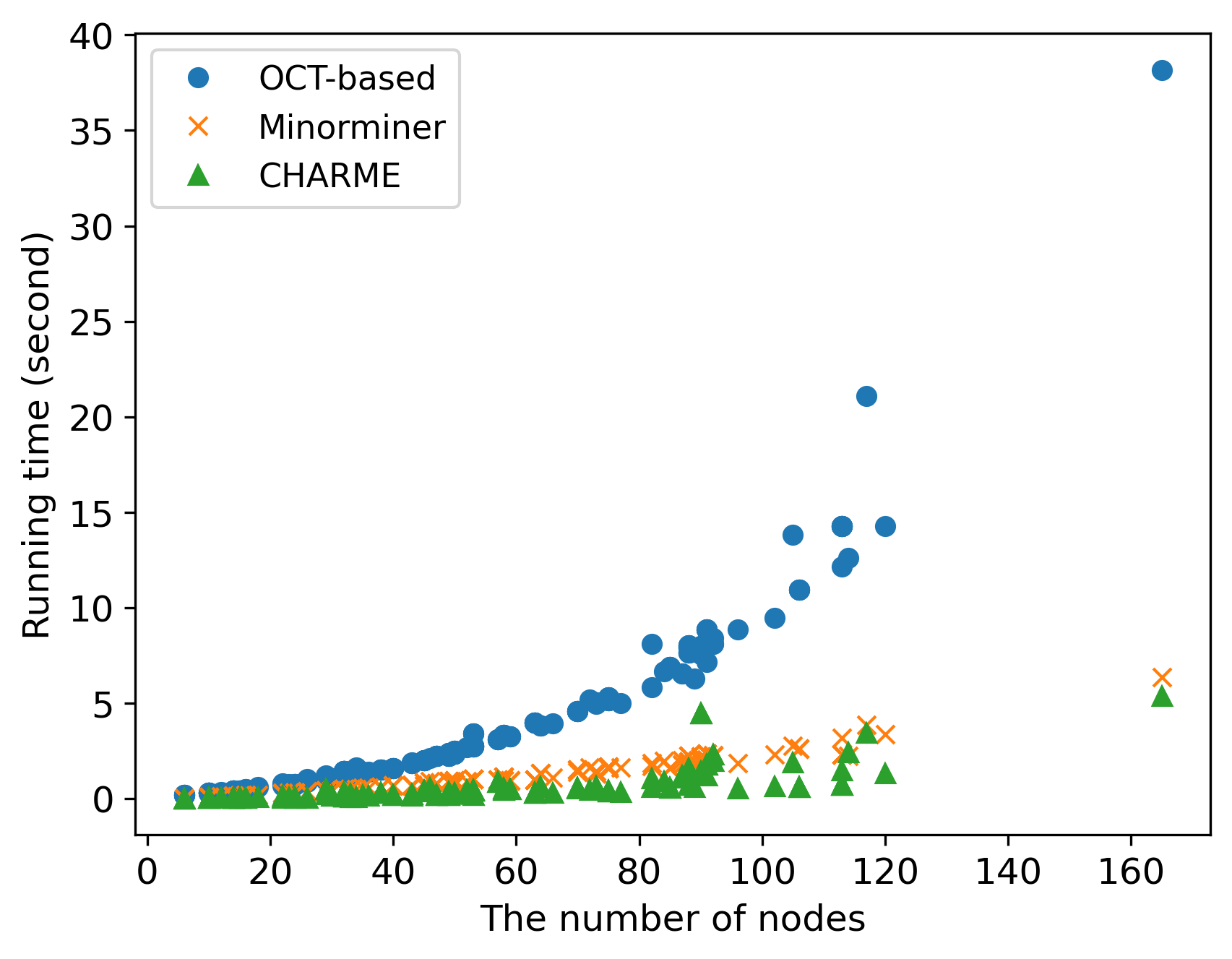}
    \caption{Comparison between three methods in terms of running time in the real world dataset with the number of nodes ranging from 6 to 165.}
    \label{fig:real world:running time}
\end{figure}

On the other hand, Figure~\ref{fig:real world:running time} presents the running time analysis of the three methods. We observe a similar trend in the running time for the real-world dataset compared to the synthetic dataset. Specifically, the running time for Minorminer and \NAMEB\ is comparable, while the OCT-based method exhibits significantly higher running times. That further confirms the efficiency of bottom-up approaches in handling the real-world dataset. In addition, the running time curve of \NAMEB\ is smoother than those of Minorminer and OCT-based, which exhibit fluctuations. This observation suggests that \NAMEB's performance is more stable compared to its competitors.

%% file: Content/6_conclusion.tex
In conclusion, our work introduces \NAMEB, a novel approach using Reinforcement Learning (RL) to tackle the minor embedding problem in Quantum Annealing (QA). Through experiments on synthetic and real-world datasets, \NAMEB\  demonstrates superior performance compared to existing methods, including fast embedding techniques such as Minorminer and ATOM, as well as the OCT-based approach known for its high-quality solutions but slower runtime. Additionally, our proposed exploration strategy enhances the efficiency of \NAMEB's training process. These results highlight the potential of \NAMEB\ to address the scalability challenges in QA, offering promising avenues for future research and application in quantum optimization algorithms.